\documentclass[a4paper, 11pt]{article}
\usepackage[english]{babel}               % [french, frenchb, english, ]
\usepackage{amsmath,amsfonts,amssymb,amsthm}
\usepackage{lmodern}        % Latin Modern family of fonts. Very much like Computer Modern, but with many more glyphs
\usepackage[T1]{fontenc}    % fontenc is oriented to output, that is, what fonts to use for printing characters.
\usepackage{libertine} \usepackage[libertine]{newtxmath}
\usepackage[utf8]{inputenc} % inputenc allows the user to input accented characters directly from the keyboard;
\usepackage{bm}             % load after all math to give access to bold math
\usepackage[margin=1in]{geometry}     % very customizable margins. Under some (rare) circumstances, should be loaded after hyperref
\usepackage[final, babel]{microtype} % many good lay-out/justification effects, see: texblog.net/latex-archive/layout/pdflatex-microtype/
\usepackage{setspace}  % set line spacing
\usepackage{tikz,pgfplots}
\usetikzlibrary{arrows,patterns,plotmarks,decorations.pathreplacing,calc}
\tikzstyle{every picture} += [>=stealth]
\pgfplotsset{compat=newest}
\usepackage{float}                      % Improved interface for floating objects ; add [H] option
\usepackage[margin=10pt,font=small,labelfont=bf]{caption}
\usepackage{subcaption}
\usepackage{multirow}
\usepackage{booktabs}
\usepackage{pgfplotstable}
\usepackage{afterpage}
\usepackage[title]{appendix}
\usepackage{enumerate}
\usepackage{algorithm}
\usepackage{algorithmicx}
\usepackage{algpseudocode}
\usepackage[normalem]{ulem}
\usepackage{sectsty}   % change section title style
\usepackage{xspace}    % generate space after \ commands
\usepackage{siunitx}
\usepackage{natbib}
\usepackage{xfrac}     % inline fraction
\usepackage[]{authblk}
\usepackage{url}
\usepackage[hypertexnames=false,hyperfootnotes=false]{hyperref}
\newcommand{\I}[1]{\ensuremath{\mathbb{I}_{\left\{#1\right\}}}} % indicator function
\newcommand{\Inb}[1]{\ensuremath{\mathbb{I}_{#1}}} % indicator function, no brackets
\newcommand{\PR}{\ensuremath{\mathsf{P}}} % probability
\newcommand{\E}{\ensuremath{\mathsf{E}}} % expectation

\newcommand{\subjectto}{\text{\textrm subject to}} % subject to
\DeclareMathOperator{\Var}{Var}

\DeclareMathOperator*{\argmin}{\mathrm{argmin}}
\DeclareMathOperator*{\argmax}{\mathrm{argmax}}

\newtheoremstyle{thm-sf}{}{}{\itshape}{}{\sffamily\bfseries}{.}{ }{}
\theoremstyle{thm-sf}
\newtheorem{theorem}{Theorem}
\newtheorem{proposition}{Proposition}

\newtheorem{lemma}{Lemma}
\theoremstyle{definition}
\newtheorem{assumption}{Assumption}

\theoremstyle{remark}
\newtheorem{remark}{Remark}

\allsectionsfont{\sffamily}
\makeatletter
\def\@seccntformat#1{\csname the#1\endcsname.\quad}
\makeatother
\floatname{algorithm}{\normalfont\sffamily\bfseries Algorithm}
\floatstyle{ruled}
\newcommand{\nc}[1]{{\color{red} \noindent {NC:} #1}}
\setcitestyle{authoryear,round,semicolon,aysep={,},yysep={,},notesep={, }}

\title{A Primal-dual Learning Algorithm for Personalized Dynamic Pricing with an Inventory Constraint}
\author[1]{Ningyuan Chen\thanks{ningyuan.chen@utoronto.ca}}
\author[2]{Guillermo Gallego\thanks{ggallego@ust.hk}}
\affil[1]{Rotman School of Management, University of Toronto}
\affil[2]{Department of Industrial Engineering \& Decision Analytics\authorcr The Hong Kong University of Science and Technology}
\date{}
\begin{document}
% todo: define empirically optimal, make entry or component consistent.
\maketitle
\begin{abstract}

We consider the problem of a firm seeking to use personalized pricing to sell an exogenously given stock of a product over a finite selling horizon to different consumer types. We assume that the type of an arriving consumer can be observed but the demand function associated with each type is initially unknown.  The firm sets personalized prices dynamically for each type and attempts to maximize the revenue over the season.  We provide a learning algorithm that is near-optimal when the demand and capacity scale in proportion.  The algorithm utilizes the primal-dual formulation of the problem and learns the dual optimal solution explicitly.  It allows the algorithm to overcome the curse of dimensionality (the rate of regret is independent of the number of types) and sheds light on novel algorithmic designs for learning problems with resource constraints.
\end{abstract}

\section{Introduction}
Dynamic pricing is practiced in many industries including travel, entertainment, and retail.
When the capacity cannot be adjusted within the sales horizon, dynamic pricing can increase revenues significantly by adjusting prices in response to the changes in the marginal value of capacity that are driven by the demand and pricing process. The presence of online channels has enabled sellers to use \emph{personalized dynamic pricing} to different consumer types resulting in potentially higher firm profits. The types may correspond to the features of a consumer, such as age, gender and address, which can be observed or inferred through membership programs and browser cookies.
Along with opportunities come challenges. The aggregate demand forecasts from historical data, which usually reflects the price sensitivity of the entire market, is of little use. Instead, the firm has to form accurate demand estimates for each type of consumers.

In this paper, we consider a firm selling a product over a finite sales horizon.
The inventory is given at the beginning of the horizon and not allowed to be replenished. The inability to order additional inventory is a hard constraint in the travel industry as it is nearly possible to add capacity to a plane or to a hotel in the short run. In fashion retailing, this is also a hard constraint as production and distribution lead-times may be larger than the sales horizon.  We assume there are $M$ different consumer types.
The types may be determined in advance by clustering algorithms, which can be applied to each consumer to label its type. We refer the reader to Chapter 8 in \citet{gallego2019revenue} for ideas on how to cluster consumer types into a reasonable number of types.  Although the type of each consumer is observed by the firm, the demand functions associated with each type are not known initially.
Therefore, the firm has to experiment different prices for each type of consumers to learn the demand functions and find the optimal prices. Therefore, it features the exploration/exploitation (learning/earning) trade-off.

We propose a learning algorithm for the personalized dynamic pricing problem described above.
Compared to the literature, our algorithm explicitly learns the dual solution, in addition to the optimal prices in the primal space.
This allows the algorithm to achieve the near-optimal regret, a measure commonly used to assess learning algorithms.

\subsection{Contributions and Insights}\label{sec:contribution}

This paper makes two contributions to the literature. By regarding the demand from the $M$ consumer types as demands for $M$ different products, then our problem is a special case of a network dynamic pricing problem with $M$ products and a single resource constraint. To the best of our knowledge, no algorithm has achieved the optimal rate of regret under the general assumptions.\footnote{The algorithm in \citet{chen2018self} achieves the same regret assuming the objective function is infinitely smooth with uniformly bounded derivatives. However, such assumption are often too restrictive.
For example, the $k$th derivative of $d(p)=\exp(-ap)$ is not uniformly bounded for $a>1$ as $k\to\infty$.}
See Section~\ref{sec:literature} and Section~\ref{sec:contribution-discussion} for more details.

From the perspective of algorithmic design, we demonstrate the feasibility of integrating the primal-dual formulation and learning.
The dual variable is not a typical target to learn in the learning literature, because unlike the primal variables, it cannot be experimented directly.
In our algorithm, we empirically estimate the Lagrangian function and sequentially form interval estimators for the \emph{dual optimal} solution.
This approach may provide novel algorithmic architectures for other learning problems with resource constraints.\footnote{In other learning algorithms involving primal-dual explorations such as \citet{badanidiyuru2013bandits}, the dual optimal solution is not learned explicitly. So the design of their algorithm is fundamentally different from ours.}

This paper provides the following qualitative insights:
\begin{itemize}
    \item It pays off to explicitly learn the dual optimal solution.
        The pricing decisions for $M$ types of consumers are coupled through the inventory constraint.
        However, having an accurate estimator for the dual optimal solution helps to decouple them into $M$ independent learning problems.
        This is the key reason why our algorithm can achieve the near-optimal regret.
    \item The learning complexity depends on not only the number of primal decision variables, but also the number of dual variables.
        As shown by \citet{besbes2012blind,slivkins2014contextual}, a high-dimensional decision vector ($M$ in this case)
        usually significantly complicates learning, reflecting the curse of dimensionality.
        \citet{slivkins2014contextual} shows that the best achievable regret for a generic learning problem without resource constraints is $n^{-1/(2+M)}$ where $M$ is the dimension of the decision vector\footnote{The rate of regret usually involves logarithmic terms. When there is no ambiguity, we omit those terms because they are dominated by the polynomial terms.}.
        In contrast, we are able to obtain the rate $n^{-1/2}$ whose exponent is independent of $M$.
        This is because the $M$ decision variables can be decoupled if the value of the dual optimal solution is given, and thus the \emph{effective dimension} of the problem is no more than the number of dual variables, which is one in our case.
\end{itemize}

\subsection{Literature Review}\label{sec:literature}
There is a stream of rapidly growing literature on a firm's pricing problem when the demand function is unknown \citep[e.g.][]{besbes2009dynamic,araman2009dynamic,farias2010dynamic,broder2012dynamic,denboer2014simul,den2015dynamicor,keskin2014dynamic,cheung2017dynamic,keskin2018incomplete,den2019discontinuous}.
See \citet{den2015dynamic} for a comprehensive survey.
Since the firm does not know the optimal price, it has to experiment different (suboptimal) prices and update its belief about the underlying demand function.
Therefore, the firm has to balance the exploration/exploitation trade-off, which is usually referred to as the learning-and-earning problem in this line of literature.
Among them, our paper is related to those with nonparametric formulations and inventory constraints \citep{besbes2009dynamic,wang2014close,lei2014near}.
In addition, we consider personalized dynamic pricing for multiple types of consumers, while most of the above papers consider a single type.

Personalized dynamic pricing can be regarded as a special case of learning with contextual information \citep{qiang2016dynamic,javanmard2016dynamic,cohen2016feature,ban2017personalized,chen2018nonparametric,keskin2019dynamic}.
The main difference of our paper from this stream of literature is summarized below.
First, instead of representing the contextual information by a feature vector, we choose to use discrete types to categorize consumers.
This could be the outcome of a clustering procedure that pre-processes consumer data.
Since the number of types is arbitrary, our setup is merely a technical simplification without losing too much practical generality.
Second, we use a nonparametric formulation for the objective function.
That is, the demand functions of each type of consumers are only required to satisfy some basic assumptions such as continuity without any specific forms.
Third, unlike this literature, we consider an inventory constraint and thus the pricing decision made over time has intertemporal dependence.

The problem studied in this paper is a special case of the multi-product dynamic pricing problem over a network \citep{gallego1997multiproduct} and thus closely related to the literature on demand learning in that setting.
\citet{besbes2012blind} study the multi-product network revenue management problem with unknown demand functions when the price for each product are chosen from a discrete set.
(Hereafter we use network revenue management to highlight the setup of a price menu, in contrast to network dynamic pricing that allows for continuous prices.)
The proposed algorithm achieves diminishing regret when the inventory and demand are scaled in proportion.
\citet{ferreira2017online} study the same problem as \citet{besbes2012blind} and show that Thompson sampling can achieve the rate of regret, $n^{-1/2}$,
which is the best one can hope for with even one product and one resource.
For continuous prices, however, \citet{besbes2012blind} demonstrate that learning may suffer from the curse of dimensionality.
The incurred regret may grow at rate $n^{-1/(d+3)}$ with $d$ products (which is equivalent to the number of types in our problem).
This is consistent with the result in \citet{slivkins2014contextual}, which studies a generic learning problem without inventory constraints.
The \emph{tight} regret bound derived in \citet{slivkins2014contextual} grows at $n^{-1/(d+2)}$ for $d$ continuous decision variables.
Sufficient smoothness may relieve the curse of dimensionality, as argued by \citet{besbes2012blind,chen2018self}.
In particular, with an infinite degree of smoothness with bounded derivatives, \citet{chen2018self} design an algorithm that almost achieves rate $n^{-1/2}$.
Global convexity helps as well, as \citet{chen2019network} propose a gradient-based algorithm that achieves rate $n^{-1/5}$.
In this paper, we present a learning algorithm that achieves the optimal rate $n^{-1/2}$ with one resource constraint and arbitrary number of products (consumer types), without imposing smoothness conditions.

This paper is also related to the vast literature studying multi-armed bandit problems. See \citet{cesa2006prediction,bubeck2012regret} for a comprehensive survey.
The classic multi-armed bandit problem involves finite arms.
There is a stream of literature studying the so-called continuum-armed bandit problems \citep{kleinberg2005nearly,Auer2007,kleinberg2008multi,bubeck2011x}, in which there are infinite number of arms (decisions).
\citet{slivkins2014contextual} provides a tight regret bound on a generic learning problem with multiple continuous decision variables;
the regret deteriorates exponentially as the number of decisions increases, demonstrating the curse of dimensionality.
This line of literature does not consider resource constraints.
Recently, there are studies combining multi-armed bandit problems with resource constraints, which is referred to as bandits with knapsacks (BwK) \citet{badanidiyuru2013bandits,badanidiyuru2014resourceful,agrawal2014bandits}.
The regret derived in those papers is not directly comparable to ours, because the decisions are discrete in their setting.
For discrete decisions, the curse of dimensionality does not emerge.
This paper is also related to online convex optimization (OCO); see \citet{shalev2012online} for a review.
It is worth pointing out that the duality approach has also been used in BwK and OCO to implement the algorithm and prove the regret bound.
However, the dual optimal solution is typically not learned explicitly.
One exception is \citet{mahdavi2013stochastic}, whose algorithm learns the dual solutions by gradient descent in the primal/dual space.
In their OCO setting, the function is given in each period and the gradient can be evaluated, which does not apply to our problem.

\section{Problem Formulation}\label{sec:problem-formulation}
Consider a monopolistic firm selling a single product in a finite selling season $T$, with $c$ units of initial inventory.
The product cannot be replenished and perishes at the end of the horizon with zero salvage values.
There are $M$ types of consumers.
Consumers with the same type have similar features such as education backgrounds, ages, and addresses.
The firm observes the type of each arriving consumer, and is allowed to price-discriminate according to the type.
This is referred to as personalized pricing, which is increasingly popular in online retailing due to the observation that the demand function differs dramatically across types.
Therefore, we model the arrival of type-$m$ consumers by a Poisson process with instantaneous rate $d_m(p_m(t))$,
where $p_m(t)$ is the price charged for type-$m$ consumers at time $t$ and $d_m(\cdot)$ is the demand function of type-$m$ consumers.

We focus on the case that the information of the demand function associated with each type and the type distribution among the population is absent at the beginning of the season.
The objective of the firm is to maximize the expected revenue collected over the horizon, subject to the inventory constraint.
To achieve the goal, the firm has to learn the demand function $d_m(\cdot)$ for $m=1,\dots,M$ and the associated optimal prices in the process.
We first characterize the problem when all the information is available.

\subsection{The full-information Benchmark}\label{sec:full_info}
When $d_m(\cdot)$ is known to the firm, the firm's objective is to maximize
\begin{align}\label{eq:full_info}
    J(T,c) = \max_{\bm p(t)\in \mathcal F_t}\quad &  \E\left[\sum_{m=1}^M\int_{t=0}^T p_m(t) d N_{m,t}(d_m(p_m(t))) \right]\\
    \subjectto \quad & \sum_{m=1}^M\int_{t=0}^T d N_{m,t}(d_m(p_m(t))) \le c,\notag
\end{align}
where $\bm p(t)=\left\{p_1(t),\ldots,p_M(t)\right\}$ is a pricing policy that is adapted to the filtration $\mathcal F_t$ associated with the sales process, and $N_{m,t}(\lambda_t)$ is an independent Poisson process with instantaneous rate $\lambda_t$.
When the inventory is depleted, then $\bm p(t)$ is forced to be $\bm p_{\infty}$, which is a menu of choke prices at which future demand from all types is turned off.

A classic approach in revenue management (e.g., see \citealt{gallego1997multiproduct}) to this problem is to consider the fluid approximation of \eqref{eq:full_info}.
That is
\begin{align} \label{eq:fluid}
        J^D(T,c) = \max_{\bm p(t)}\quad &  \sum_{m=1}^M\int_{t=0}^T p_m(t) d_m(p_m(t)) dt \\
    \subjectto \quad & \sum_{m=1}^M\int_{t=0}^T d_m(p_m(t))dt \le c,\notag
\end{align}
where we have replaced the Poisson process $N_{m,t}(d_m(p_m(t))$ in the original formulation by the intensity $d_m(p_m(t))$.
Note that the fluid approximation \eqref{eq:fluid} is a deterministic optimization problem.
Before presenting the primal-dual formulation of \eqref{eq:fluid}, we make the following standard assumption:
\begin{assumption}\label{asp:convexity}
    For a price domain $p\in [0,p_{\infty}]$ and $m=1,\ldots,M$, we assume
    \begin{enumerate}
        \item The demand $d_m(p)$ is strictly decreasing with an inverse function $d^{-1}_m(\cdot)$ and bounded with $d_m(p)\le M_1$.
        \item Define the revenue rate as a function of the demand rate $\lambda$, $r_m(\lambda)\triangleq\lambda d^{-1}_m(\lambda)$. The functions $r_m(\lambda)$, $d_m(p)$ and $d^{-1}_m(\lambda)$ are Lipschitz continuous with factor $M_2$.
        \item $r(\lambda)$ is twice-differentiable and strictly concave, $0<M_3\le -r_m''(\lambda)\le M_4$.
    \end{enumerate}
\end{assumption}
\begin{remark}
    Assumption~\ref{asp:convexity} implies that both $d_m(p)$ and $d_m^{-1}(\lambda)$ are differentiable. The derivatives are bounded by the interval $[-M_2, -1/M_2]$.
\end{remark}
It is easy to verify that this mild assumption is satisfied by most demand functions, such as the exponential demand $d(p)=a\exp(-bp)$ and the linear demand $d(p)=a-bp$.
The exponential demand, on the other hand, does not have uniformly bounded derivatives of any orders, required by \citet{chen2018self}.

\subsubsection*{Primal-dual Formulation}
Consider the Lagrangian function for the fluid approximation \eqref{eq:fluid}
\begin{equation}\label{eq:lagrangian}
    L(\bm p(t), z) = cz +\sum_{m=1}^M\int_{t=0}^T (p_m(t)-z) d_m(p_m(t)) dt
\end{equation}
and the dual function
\begin{equation}\label{eq:dual}
    g(z) = \max_{\bm p(t)}\left\{L(\bm p(t), z)\right\}=cz+ \max_{\bm p(t)}\left\{ \sum_{m=1}^M\int_{t=0}^T (p_m(t)-z) d_m(p_m(t)) dt\right\}.
\end{equation}
Under Assumption~\ref{asp:convexity}, the following quantities are well defined
\begin{equation}
    \mathcal R_m(z)\triangleq \max_{p}\left\{ d_m(p)(p-z)\right\}\quad \text{and}\quad \mathcal P_m(z)\triangleq \argmax_{p}\left\{ d_m(p)(p-z)\right\}.
\end{equation}
$\mathcal R_m(z)$ and $\mathcal P_m(z)$ can be interpreted as the optimal value and optimal solution of the profit-maximization problem for type-$m$ consumers when the unit cost is $z$.
They are closely related to the dual function \eqref{eq:dual} as $g(z)=cz+T\sum_{m=1}^M\mathcal R_m(z)$;
provided with a dual variable $z$, the optimal $\bm p(t)$ in \eqref{eq:dual} is time-invariant: $p_m(t)\equiv \mathcal P_m(z)$.
The properties are summarized below (see also \citealt{gallego2019revenue}):

\begin{proposition}\label{prop:dual}
    Under Assumption~\ref{asp:convexity}, we have
    \begin{enumerate}
        \item $\mathcal P_m(z)$ is increasing in $z$ and $\mathcal P'_m(z)$ is bounded.
        \item $\mathcal R_m(z)$ is decreasing and convex in $z$; $\mathcal R_m'(z)=-\sum_{m=1}^M d_m(\mathcal P_m(z))$.
        \item $g(z)$ is twice differentiable and strictly convex.
        \item Let $z^{\ast}\triangleq\argmin_{z\ge 0} \{g(z)\}$. The optimal solution to \eqref{eq:fluid} is $p_m(t)\equiv p^{\ast}_m\triangleq\mathcal P_m(z^{\ast})$.
            Moreover, complementary slackness holds: $z^{\ast}(c-T\sum_{m=1}^Md_m(\mathcal P_m(z^{\ast})))=0$.
    \end{enumerate}
\end{proposition}
\begin{remark}\label{rmk:dual_convexity}
    To simplify the notation, we use the same set of constants as in Assumption~\ref{asp:convexity} and assume that $0<M_3\le g''(z)\le M_4$; $\mathcal P_m(z)$ is Lipschitz continuous with factor $M_2$.
\end{remark}
Proposition~\ref{prop:dual} states that the fluid approximation~\eqref{eq:fluid} admits a time-invariant pricing policy $p^{\ast}_m$.
Moreover, the optimal solution is closely related to the dual optimal solution $z^{\ast}$, which is usually interpreted as the shadow cost of inventory.
When $z^{\ast}>0$, the initial capacity is insufficient and thus the inventory constraint is binding by complementary slackness: $c=T\sum_{m=1}^Md_m(\mathcal P_m(z^{\ast}))$.
When $z^{\ast}=0$, the inventory is sufficient and the optimal price $p^{\ast}_m=\mathcal P_m(0)$ maximizes the revenue rate $pd_m(p)$ as if there is no inventory constraint.

%Therefore, we have $g(z)= cz+T\mathcal R_m(z)$ and
%the dual problem yields a solution
%%Given Assumption~\ref{asp:convexity}, we have
%%\begin{proposition}\label{prop:primal-dual}
%%    \begin{enumerate}
%%        \item
%%        \item $g(z)$ is strictly convex,
%%    \end{enumerate}
%%\end{proposition}
%\todo{I believe the following is a corollary of Assumption~\ref{asp:convexity}, but maybe we need a bit more technical conditions to ensure strict convexity, maybe increasing failure rate?}
%Therefore, solving $\max_{z\ge 0}g(z)$ yields the dual optimal solution $z^{\ast}$, which in turn gives the primal optimal solution to \eqref{eq:fluid}: $p_m^{\ast}= p_m(z^\ast)$.

\subsubsection*{Scaled Demand and Capacity}
The connection between $J^D(T,c)$ and $J(T,c)$ has been studied extensively in the
revenue management literature.
In particular, \citet{gallego1997multiproduct} find that the revenue from the fluid approximation is an upper bound for the stochastic problem,
i.e., $J^D(T,c)\ge J(T,c)$.
The tie becomes closer when the demand and capacity scale in proportion:
if we index a sequence of systems by $n$ and let $d_{m,n}(\cdot)=nd_m(\cdot)$ and $c_n = nc$ in the $n$th system, then the revenues satisfy $J^D_{n}(T,c)-J_n(T,c)=O(\sqrt{n})$.
Since $J^D_{n}(T,c)$ scales linearly in $n$,
the gap between $J_n^D(T,c)$ and $J_n(T,c)$ diminishes relative to the earned revenue as $n$ grows.
More importantly, the optimal solution to \eqref{eq:fluid}, $\{p^{\ast}_m\}_{m=1}^M$, which also maximizes the fluid approximation for the scaled system $J^D_n(T,c)$, performs well in the stochastic problem \eqref{eq:full_info} as a special suboptimal pricing policy (it is constant and thus adapted to $\mathcal F_t$).
More precisely, the expected revenue for $\{p^{\ast}_m\}_{m=1}^M$ in the $n$th stochastic system satisfies
\begin{equation*}
    J^D_n(T,c)-\E\left[\sum_{m=1}^M\int_{t=0}^\tau p_m^{\ast} d N_{m,t}(nd_m(p_m^{\ast})) \right]=O(\sqrt{n}),
\end{equation*}
where $\tau$ is the minimum of $T$ and the stopping time when the inventory reaches zero.
Combined with the fact that $J_n^D\ge J_n$, $\{p^{\ast}_m\}_{m=1}^M$ is near-optimal in the $n$th stochastic system.
Therefore, the prices $\{p^{\ast}_m\}_{m=1}^M$ are the goal of our learning policy when $d_m(\cdot)$ is not known to the firm.

\subsection{The Learning Policy and the Target Regret}\label{sec:regret}
Suppose the firm does not know $d_m(\cdot)$ at the beginning of the horizon.
To earn high expected revenues over the horizon, the firm adopts an $\mathcal F_t$-predictable pricing policy $\pi$.
That is, at time $t$, $\pi_t$ only depends on the adopted prices and the observed sales for each type of consumers prior to $t$.
It then outputs a price vector $\{P_1(t),\ldots,P_m(t)\}$ for each type of consumers at time $t$.
We denote the expected revenue associated with a policy $\pi$ by $J^{\pi}(T,c)$.
Clearly, the unavailability of the information regarding $d_m(\cdot)$ incurs a cost to the firm, and thus $J^{\pi}(T,c)\le J(T,c)\le J^D(T,c)$.

The objective of this study is to design a policy so that $J(T,c)-J^{\pi}(T,c)$ is small, especially when demand and capacity are scaled.
Therefore, similar to \citet{besbes2009dynamic}, we consider the following criterion, referred to as the regret, of a policy $\pi$:
\begin{equation}\label{eq:regret}
    R_n^{\pi}(T,c) = 1- \frac{J_n^{\pi}(T,c)}{J_n^D(T,c)},
\end{equation}
where $J_n^{\pi}(T,c)$ is the expected revenue $\pi$ generates in the $n$th stochastic system.
Note that $\pi$ may depend on $n$, and we suppress the dependence to simplify notations.
The regret measures the revenue loss $J_n^D(T,c)-J_n^{\pi}(T,c)$ relative to $J_n^D(T,c)$.
The goal of the policy $\pi$ is to ensure $\lim_{n\to\infty}R_n^{\pi}(T,c)=0$.
That is, the learning incurs no significant cost for large systems.

It has been shown in \citet{besbes2009dynamic,wang2014close} that for $M=1$, any learning policy
incurs regret whose rate is no less than $n^{-1/2}$ for some problem instances.
Indeed, even if we replace $J_n^{\pi}(T,c)$ by $J_n(T,c)$, which is the full-information upper bound for $J_n^{\pi}(T,c)$, the quantity \eqref{eq:regret} is of order $n^{-1/2}$ by the discussion in Section~\ref{sec:full_info}.
In other words, one cannot expect to design a learning policy whose regret grows slower than $n^{-1/2}$.
Thus, $n^{-1/2}$ (possibly with logarithmic terms in $n$) is the target regret of our proposed learning policy.
It is also worth noticing that we adopt a nonparametric formulation.
In parametric formulations, one may estimate the parameter using historical data and conduct a policy that maximizes the objective based on the estimator, without deliberate exploration.
The so-called certainty-equivalence control may or may not lead to incomplete learning as shown by \citet{keskin2018incomplete}.
In the nonparametric setting, a wider range of exploration seems mandatory because no global information is learned through local experimentation.

\section{The Primal-dual Learning Algorithm} \label{sec:algorithm}
In this section, we introduce an algorithm (learning policy) based on the primal-dual formulation.
We first explain the steps of the algorithm, and then analyze its regret.
Combined with the lower bound for regret in \citet{besbes2009dynamic,wang2014close},
the regret of the algorithm achieves near optimality for the problem considered in Section~\ref{sec:problem-formulation}.

Before proceeding, we state what the firm has information of initially.
The firm knows $M$, $T$, $n$, $c$, and the constants specified by Assumption~\ref{asp:convexity}.
Moreover, we impose a mild assumption in addition to Assumption~\ref{asp:convexity}.
\begin{assumption}\label{asp:initial}
    There exist intervals $[\underline p,\overline p]$ and $[0,\overline z]$ such that
    $p_{m}^{\ast}\in (\underline p,\overline p)$ and $z^{\ast}\in (0,\overline z)$ for all $m$.
    Moreover, $\left\{\mathcal P_{m}(z): z\in [0,\overline z]\right\}\subset [\underline p,\overline p]$.
    The intervals $[\underline p,\overline p]$ and $[0,\overline z]$ are known to the firm.
\end{assumption}
Note that $p^{\ast}_m$ and $z^{\ast}$ are the primal/dual optimal solutions to the fluid approximation \eqref{eq:fluid}.
Assumption~\ref{asp:initial} states that although the firm does not know the optimal solutions, it does have the information of their ranges.
Since $\underline p,\overline p$ and $\overline z$ can be arbitrary finite numbers,
this assumption is not restrictive.

\subsection{The Intuition}
We first explain the intuition behind the algorithm.
If the firm knew the full information, then it would have found the pricing policy $p_m^\ast$
through the primal-dual formulation
\begin{align}\label{eq:primal_dual_optimization}
    z^{\ast}&=\argmin_{z\ge 0}\left\{g(z)\right\}=\argmin_{z\ge 0}\left\{cz+ T\sum_{m=1}^M \mathcal R_m(z)\right\}\\
%    \implies 0&= c-T\sum_{m=1}^M d_m(\mathcal P_m(z^{\ast}))\\
p^{\ast}_m &= \mathcal P_m(z^{\ast})\quad\forall m=1,\dots,M.\notag
\end{align}
When the information of $d_m(p)$ is not available, both optimization problems are unsolvable.
However, the firm can experiment with different prices and use the observed sales
as a noisy but unbiased estimator for $d_m(p)$ at those prices.
The noisy estimator is a Poisson random variable.
Then, the firm could plug the estimators into \eqref{eq:primal_dual_optimization} to solve them ``empirically'',
obtaining estimators for $z^{\ast}$ and $p_m^{\ast}$ for $m=1,\ldots,M$.
One would imagine that the estimators for $z^{\ast}$ and $p_m^{\ast}$ are not necessarily accurate.
Indeed, the accuracy of such a procedure depends crucially on two aspects:
\begin{enumerate}
    \item The length of the period during which the price is experimented.
        The longer the period, the less noisy the estimators for $d_m(p)$ are.
    \item The granularity of the experimented prices. The estimator for $d_m(p)$ is based on a discrete set of prices. It inevitably incurs discretization error in order to solve a continuous optimization problem \eqref{eq:primal_dual_optimization}.
\end{enumerate}
Ideally, to obtain accurate estimators for $z^{\ast}$ and $p_m^{\ast}$, the firm would
set a refined grid of prices for each type of consumers and try each price for a long period during the season.
Those suboptimal prices, however, lead to substantial revenue loss.

To solve the exploration/exploitation dilemma, we divide $[0,T]$ into multiple phases.
After each phase, the sales for each type of consumers during the phase are observed at a set of prices to form estimators for $d_m(p)$.
Then \eqref{eq:primal_dual_optimization} is solved empirically to obtain point estimators for $z^{\ast}$ and $p_m^{\ast}$.
In the next phase, those point estimators are used to form \emph{interval estimators} for $z^{\ast}$ and $p_m^{\ast}$.
The interval estimators help to narrow down the range of prices to experiment.
Therefore, as the algorithm enters new phases, the burden to explore is gradually relieved and it can afford to try a more refined price grid for a longer period of time.
The revenue loss is also limited because the experimented prices fall into a narrow interval containing $p_m^{\ast}$ with high probability.

\subsection{Description of the Algorithm}\label{sec:description}
Next we explain the details of the algorithm.
Let $\{P_1(t),\ldots,P_M(t)\}$ be the stochastic pricing policy associated with $\pi$.
Without further mention, we always suppose that when the inventory is depleted, $P_m(t)$ is automatically switched to the choke price $p_{\infty}$ for all $m$.
Given $n$, the algorithm divides $[0,T]$ into consecutive phases $k=1,2,\ldots,K$.
The length of phase $k$ is $\tau^{(k)}$.
We also denote the beginning of phase $k$ by $t_k$. Thus $t_k = \sum_{i=1}^{k-1}\tau^{(i)}$.
Let $\epsilon>0$ be a small constant.

At the beginning of phase $k$, the firm has interval estimators for $p_m^{\ast}$ and $z^{\ast}$, $[\underline p_{m}^{(k)},\overline p_{m}^{(k)}]$ and $[\underline{z}^{(k)},\overline z^{(k)}]$, obtained from the last phase, which ensure that $p^{\ast}_m\in [\underline p_{m}^{(k)},\overline p_{m}^{(k)}]$ and $z^{\ast}\in [\underline{z}^{(k)},\overline z^{(k)}]$ with high probability.
During phase $k$, the price interval $[\underline p_{m}^{(k)},\overline p_{m}^{(k)}]$, whose length is denoted $\Delta_m^{(k)}$, is discretized to $N^{(k)}+1$ equally spaced grid points, i.e., $p_{m,j}^{(k)}\triangleq \underline p_{m}^{(k)}+j \delta_m^{(k)}$ for $j=0,\ldots,N^{(k)}$ and $\delta_m^{(k)}\triangleq\Delta^{(k)}_m/N^{(k)}$.
During phase $k$, the algorithm sets price $p_{m,j}^{(k)}$ for type-$m$ consumers for a period of length $\tau^{(k)}/(N^{(k)}+1)$.

At the end of phase $k$,
the observed sales, $D_{m,j}^{(k)}$, from type-$m$ consumers at price $p_{m,j}^{{(k)}}$ is a Poisson random variable with mean $nd_m(p_{m,j}^{{(k)}}) \tau^{(k)}/(N^{(k)}+1)$.
Therefore, an unbiased estimate for $d_m(p_{m,j}^{(k)})$ is
\begin{equation*}
    \hat d^{(k)}_{m,j} \triangleq  \frac{N^{(k)}+1}{n\tau^{(k)}}D_{m,j}^{(k)}.
\end{equation*}
To form a point estimator for $z^{\ast}$, the firm substitutes $\hat d^{(k)}_{m,j}$ into $g(z)$ (the right-hand side of the first equation of \eqref{eq:primal_dual_optimization}), i.e.,
\begin{equation*}
    g(z) = cz+T\sum_{m=1}^M\max_{p_m}d_m(p_m)(p_m-z)\approx cz + T\sum_{m=1}^M\max_{j_m = 0,\ldots, N^{(k)}} \hat d_{m,j_m}^{(k)}(p_{m,j_m}^{(k)}-z).
\end{equation*}
To find $z\in [\underline{z}^{(k)},\overline z^{(k)}]$ that maximizes the above expression,
the firm divides $[\underline{z}^{(k)},\overline z^{(k)}]$, whose length is denoted $\Delta_z^{(k)}$, into $N_z^{(k)}$ equally spaced grid points, $\underline{z}^{(k)}+j \delta_z^{(k)}$ for $j=0,\ldots, N_z^{(k)}$ and $\delta_z^{(k)}\triangleq\Delta_z^{(k)}/(N_z^{(k)}+1)$.
Therefore, a point estimator for $z^{\ast}$ can be obtained as follows\footnote{Alternatively, the firm can find $z^{(k)\ast}$ by solving the first-order condition for $g(z)$, $c-T\sum_{m=1}^M d_m(\mathcal P_m(z))=0$, using the empirical version of $d_m$ and $\mathcal P_m$. The regret analysis holds for this case. Also note that in theory we can find the optimal $z^{(k)\ast}$ exactly without discretization, as the equation is solved offline. The discretization is for practical purposes.}:
\begin{equation}\label{eq:dual_algorithm}
    z^{(k)\ast} \triangleq  \argmin_{z\in \left\{\underline{z}^{(k)}+i \delta_z^{(k)}\right\}_{i=0}^{N_z^{(k)}}}\left\{ cz + T\sum_{m=1}^M\max_{j_m = 0,\ldots, N^{(k)}} \hat d_{m,j_m}^{(k)}(p_{m,j_m}^{(k)}-z)\right\}.
\end{equation}
%That is, we solve the dual problem $\min_{z\ge 0} g(z)$ for the discrete primal/dual grid points and empirical estimates for $d_m(\cdot)$.
Based on $ z^{(k)\ast}$, the firm can obtain point estimators for $p_m^{\ast}$ by the second equation in \eqref{eq:primal_dual_optimization}:
\begin{equation}\label{eq:primal_algorithm}
    p_{m}^{(k)\ast}\triangleq p_{m,j^{\ast}_m}^{(k)},\quad\text{where}\quad  j_m^{\ast} = \argmax_{j_m = 0,\ldots, N^{(k)}} \hat d_{m,j_m}^{(k)}(p_{m,j_m}^{(k)}-z^{(k)\ast}).
\end{equation}
This completes the procedure in phase $k$.

At the beginning of phase $k+1$, the firm constructs interval estimators $[\underline p_{m}^{(k+1)},\overline p_{m}^{(k+1)}]$ ($[\underline{z}^{(k+1)},\overline z^{(k+1)}]$) based on the point estimators $p_m^{(k)\ast}$ ($z^{(k)^{\ast}}$) and pre-specified width $\bar\Delta^{(k+1)}$ ($\bar\Delta_z^{(k+1)}$) for all $m$:
\begin{align}\label{eq:interval_est}
    \underline p_{m}^{(k+1)} = \max \left\{\underline p, p_m^{(k)\ast}- \frac{\bar\Delta^{(k+1)}}{2}\right\},&\quad \overline p_{m}^{(k+1)} = \min \left\{\overline p, p_m^{(k)\ast}+ \frac{\bar\Delta^{(k+1)}}{2}\right\}\notag\\
    \underline z^{(k+1)} = \max \left\{0, z^{(k)\ast}- \frac{\bar\Delta^{(k+1)}_z}{2}\right\},&\quad \overline z^{(k+1)} = \min \left\{\overline z, z^{(k)\ast}+ \frac{\bar\Delta^{(k+1)}_z}{2}\right\}.
\end{align}
Note that the intervals are properly truncated by $[\underline p,\overline p]$ and $[0,\overline z]$,
and this is the only reason why $\bar \Delta^{(k+1)}$ ($\bar \Delta^{(k+1)}_z$) can potentially be different from $\Delta_m^{(k+1)}$
($\Delta_z^{(k+1)}$).
Then the procedure is repeated for phase $k+1$.

In the last phase, phase $K$, the algorithm behaves differently after forming the interval estimators $[\underline p_{m}^{(K)},\overline p_{m}^{(K)}]$ and $[\underline{z}^{(K)},\overline z^{(K)}]$.
At the beginning of phase $K$, the firm checks whether $0\in [\underline{z}^{(K)},\overline z^{(K)}]$.
If so, then with high probability $z^{\ast}=0$ and the capacity is sufficient.
Therefore, the price $p_m^{\ast}$ is the unconstrained maximizer of $pd_m(p)$, i.e., $\mathcal P_m(0)$, for all $m$.
As we will show in Section~\ref{sec:parameter}, the width of the interval estimator $[\underline p_{m}^{(K)},\overline p_{m}^{(K)}]$ is roughly $\bar\Delta^{(K)}\sim n^{-1/4}$.
Therefore, if the firm adheres to a constant price $p_m\in [\underline p_{m}^{(K)},\overline p_{m}^{(K)}]$ for type-$m$ consumers, the relative revenue loss for type-$m$ consumers in phase $K$ (ignoring the random fluctuation of Poisson arrivals) is approximately
\begin{equation*}
  |p^{\ast}_md_m(p^{\ast}_m)-pd_m(p)|\sim (d_m(p_m^{\ast})-d_m(p_m))^2\sim (p_m^{\ast}-p_m)^2\sim (\bar\Delta^{(K)})^2\sim n^{-1/2},
\end{equation*}
where we rely on the concavity in Assumption~\ref{asp:convexity} and the fact that $p_m^{\ast}\in [\underline p_{m}^{(K)},\overline p_{m}^{(K)}]$ with high probability.
This meets the target regret in Section~\ref{sec:regret}.
Motivated by the argument above,
the algorithm simply charges a constant price $p^{(K)}_m = \overline p_{m}^{(K)}+ \alpha$ for type-$m$ consumers until the end of the season for a pre-specified parameter $\alpha$.
Note that we slightly mark up the prices by a small adjustment $\alpha\sim n^{-1/4}$ to guarantee that the inventory is sufficient when the inventory just meets the unconstrained optimal prices, i.e., $c=T\sum_{m=1}^Md_m(p_m^{\ast})$.

%This is because phase $K$ is the last phase and it is devoted to exploitation instead of exploration. the firm first checks if
If the firm finds $0\notin [\underline{z}^{(K)},\overline z^{(K)}]$ at the beginning of phase $K$, which implies that $z^{\ast}>0$ and the capacity is insufficient with high probability, then a different procedure has to be used in phase $K$.
The method for the case of $z^{\ast}=0$ no longer works: because $p_m^{\ast}$ is no longer the unconstrained maximizer of $pd_m(p)$, even for $p_m,p_m^{\ast}\in [\underline p_{m}^{(K)},\overline p_{m}^{(K)}]$, we have
\begin{equation}\label{eq:one-four-convergence}
  |p^{\ast}_md_m(p^{\ast}_m)-pd_m(p)|\sim |p_m^{\ast}-p_m|\sim |\bar\Delta^{(K)}|\sim n^{-1/4}.
\end{equation}
This implies that a constant price in $[\underline p_{m}^{(K)},\overline p_{m}^{(K)}]$ for type-$m$ consumers will fail to meet the target regret.
To address the problem,
let $p_m^l = \underline{p}_m^{(K)}-\alpha$ and $p_m^u = \overline{p}^{(K)}_m+\alpha$
be conservative lower and upper bounds for $p_m^{\ast}$.
It makes sure that $p_m^\ast\in [p_m^l,p_m^u]$ along with buffers around the boundary.
The buffer guarantees that the solution to the linear interpolation \eqref{eq:theta} below is stable.
Let $S(t)$ be the cumulative sales aggregated from all types of consumers up to time $t$.
The algorithm in phase $K$ is divided into the following four steps:
\begin{enumerate}
    \item For $t\in(t_K,t_K+(\log n)^{-\epsilon}]$, apply $p_m^l$ to type-$m$ consumers.
        Record the aggregate sales rate by $D^{(K)}_l$.
        That is
        \begin{equation}\label{eq:rate_Dl}
            D^{(K)}_l \triangleq (\log n)^{\epsilon}\sum_{m=1}^M \int_{t_K}^{t_K+(\log n)^{-\epsilon}} dN_{m,t}(nd_m(p_m^l)).
        \end{equation}
        Clearly, $D^{(K)}_l$ is an unbiased estimator for $n\sum_{m=1}^Md_m(p_m^l)$.
    \item For $t\in(t_K+(\log n)^{-\epsilon},t_K+2(\log n)^{-\epsilon}]$, apply $p_m^u$ to type-$m$ consumers.
        Record the aggregate sales rate $D^{(K)}_u$:
        \begin{equation}\label{eq:rate_Du}
            D^{(K)}_u \triangleq (\log n)^{\epsilon}\sum_{m=1}^M \int_{t_K+(\log n)^{-\epsilon}}^{t_K+2(\log n)^{-\epsilon}} dN_{m,t}(nd_m(p_m^u)),
        \end{equation}
        which is an unbiased estimator for $n\sum_{m=1}^Md_m(p_m^u)$.
    \item At $t=t_K+2(\log n)^{-\epsilon}$, solve $\theta\in[0,1]$ from
        \begin{equation}\label{eq:theta}
            (T-t_K)(\theta D^{(K)}_l+(1-\theta)D^{(K)}_u) = nc - S(t_K).
        \end{equation}
        If the solution $\theta\notin[0,1]$ (which will be shown to have negligible probability), then we project it to $[0,1]$.
        To interpret $\theta$, note that $nc-S(t_K)$ is the remaining inventory at $t_K$, the beginning of phase $K$.
        If $D^{(K)}_l$ and $D^{(K)}_u$ were equal to their means, $n\sum_{m=1}^M d_m(p_m^l)$ and $n\sum_{m=1}^M d_m(p_m^u)$,
        then in a fluid system starting from $t_K$ with inventory $nc-S(t_K)$,
        applying $p_m^l$ for type-$m$ consumers for a period of length $\theta(T-t_K)$ and $p_m^u$ for a period of length $(1-\theta)(T-t_K)$ would make the inventory reach zero right at $T$, according to \eqref{eq:theta}.
    \item For $t\in (t_K+2(\log n)^{-\epsilon},T]$, apply $p_m^l$ for a period of length $\theta(T-t_K)-(\log n)^{-\epsilon}$, and $p_m^u$ for a period of length $(1-\theta)(T-t_K)-(\log n)^{-\epsilon}$ until $T$.
\end{enumerate}
The goal of the above steps is to ensure the deviation of $S(T)$ from $nc$ is relatively small.
In particular, from Lemma~\ref{lem:inventory-XT} in Section~\ref{sec:analysis}, the steps guarantee $|S(T)-nc|\sim n^{-1/2}$.
%if we ignore the inventory constraint so that algorithm does not have to set choke prices once $S(t)$ reaches $nc$, then we would .
Without further exploring the price space\footnote{Recall that the interval estimators for $p_m^{\ast}$, $\bar\Delta^{(K)}\sim n^{-1/4}$, are still too wide to meet the target regret.}
the algorithm can still meet the target regret with a little exploration on the aggregate demand rate and by controlling the aggregate sales at $T$.
We will discuss this point in Section~\ref{sec:discuss}.
The notations are summarized in Table~\ref{tab:notation} in the appendix.
The detailed steps of the algorithm are demonstrated in Algorithm~\ref{alg:primal-dual}.
Note that both ``Input'' and ``Constant'' are known to the firm, while ``Parameters'' are computed in Section~\ref{sec:parameter}.
\begin{algorithm}
    \caption{The Primal-dual Learning Algorithm}
    \label{alg:primal-dual}
    \begin{algorithmic}[1]
        \State Input: $n$, $c$, $T$, $M$
        \State Constants: $M_1$, $M_2$, $M_3$, $M_4$, $\underline p$, $\overline p$, $\overline z$
        \State Parameters: $\epsilon$, $\alpha$, $K$, $\{\tau^{(k)}\}_{k=1}^{K-1}$, $\{\bar\Delta^{(k)},\bar\Delta^{(k)}_z,N^{(k)},N_z^{(k)}\}_{k=1}^{K}$\label{step:parameter}
        \State Initialize: $\underline p_{m}^{(1)}=\underline p$, $\overline p_{m}^{(1)}=\overline p$, $\underline{z}^{(1)}=0$, $\overline z^{(1)}=\overline z$ \label{step:initialize}
        \For{$k=1$ to $K-1$}\label{step:for-loop-begin}
        \State $t_k\gets\sum_{i=1}^{k-1}\tau^{(i)}$
        \Comment{The start of phase $k$}
        \State $\Delta_m^{(k)}\to \overline p_{m}^{(k)}-\underline p_{m}^{(k)}$ and $\delta_m^{(k)}\gets\Delta^{(k)}_m/(N^{(k)}+1)$ for $m=1,\ldots,M$
        \For{$i=0$ to $N^{(k)}$}
        \For{$t=t_k+ i \frac{\tau^{(k)}}{N^{(k)}+1}$ to $t_k+ (i+1) \frac{\tau^{(k)}}{N^{(k)}+1}$}
        \State Charge price $p_{m,i}^{(k)}\gets\underline p_{m}^{(k)}+i \delta_m^{(k)}$ to type-$m$ consumers
        \State Record the observed sales $D_{m,i}^{(k)}$ for $p_{m,i}^{(k)}$ for $m=1,\ldots,M$
        \EndFor
        \State $\hat d^{(k)}_{m,i} \gets  \frac{N^{(k)}+1}{n\tau^{(k)}}D_{m,i}^{(k)}$ for $m=1,\ldots,M$
        \Comment{Empirical estimate for $d_m(p_{m,i}^{(k)})$}
        \EndFor\label{step:finish-test-price}
        \State $\Delta_z^{(k)}\gets \overline z^{(k)}-\underline z^{(k)}$ and $\delta_z^{(k)}=\Delta_z^{(k)}/(N_z^{(k)}+1)$
        \State Obtain $z^{(k)\ast}$ according to \eqref{eq:dual_algorithm}\label{step:dual}
        \Comment{Estimate the dual optimal solution}
        \State Obtain $p_{m}^{(k)\ast}$ according to \eqref{eq:primal_algorithm} for $m=1,\ldots,M$
        \Comment{Estimate the primal optimal solution}
        \State Obtain $\underline z^{(k+1)}$, $\overline z^{(k+1)}$, $\underline p_{m}^{(k+1)}$, $\overline p_{m}^{(k+1)}$ according to \eqref{eq:interval_est} for $m=1,\ldots,M$
        \Comment{Obtain the interval estimators}
        \EndFor
        \State $t_K\gets\sum_{i=1}^{K-1}\tau^{(i)}$
        \Comment{The beginning of phase $K$}
        \If{$0\in [\underline{z}^{(K)},\overline z^{(K)}]$}\label{step:sufficient_capacity}
        \Comment{Sufficient capacity}
        \For{$t=t_K$ to $T$}
        \State Charge price $p^{(K)}_m \gets\overline p_{m}^{(K)}+ \alpha$ to type-$m$ consumers
        \EndFor
        \Else \label{step:insufficient_capacity}
        \Comment{Insufficient capacity}
        \State $p_m^l \gets \underline{p}_m^{(K)}-\alpha$ and $p_m^u \gets \overline{p}^{(K)}_m+\alpha$\label{step:pml_pmu}
        \For{$t=t_K$ to $t_K+(\log n)^{-\epsilon}$}
        \State Charge price $p_m^l$ to type-$m$ consumers \label{step:p_ml}
        \State Record the aggregated sales rate $D^{(K)}_l$ according to~\eqref{eq:rate_Dl}
        \EndFor
        \For{$t=t_K+(\log n)^{-\epsilon}$ to $t_K+2(\log n)^{-\epsilon}$}
        \State Charge price $p_m^u$ to type-$m$ consumers  \label{step:p_mu}
        \State Record the aggregated sales rate $D^{(K)}_u$ according to~\eqref{eq:rate_Du}
        \EndFor
        \State Let $\theta$ be the projection of the solution to \eqref{eq:theta} to $[0,1]$\label{step:theta}
        \For{$t=t_K+2(\log n)^{-\epsilon}$ to $t_K+\theta(T-t_K)+(\log n)^{-\epsilon}$}
        \State Charge price $p_m^l$ to type-$m$ consumers \label{step:p_ml_theta}
        \EndFor
        \For{$t=t_K+\theta(T-t_K)+(\log n)^{-\epsilon}$ to $T$}
        \State Charge price $p_m^u$ to type-$m$ consumers\label{step:p_mu_theta}
        \EndFor
        \EndIf
    \end{algorithmic}
\end{algorithm}

\subsection{Choice of Parameters}\label{sec:parameter}
Let $\epsilon$ be a sufficiently small constant (independent of $n$).
We set the following parameter values in Step~\ref{step:parameter}:
\begin{align*}
    \alpha & = (\log n)^{1+9\epsilon}n^{-1/4}\\
    \tau^{(k)}& = n^{- (1/2) (3/5)^{k-1}} (\log n)^{1+15\epsilon}\quad \text{for } k\le K-1\\
    \bar\Delta^{(k)} &= n^{-(1/4)(1-(3/5)^{k-1})}\\
    \bar\Delta^{(k)}_z&= n^{-(1/4)(1-(3/5)^{k-1})}(\log n)^{-2\epsilon}\\
    N^{(k)} & = n^{(1/10)(3/5)^{k-1}}(\log n)^{3\epsilon}\\
    N^{(k)}_z &= n^{(1/10)(3/5)^{k-1}}(\log n)^{\epsilon}\\
    K & = \min\left\{k: (\bar \Delta^{(k)})^2\le n^{-1/2}(\log n)^{2+16\epsilon}\right\}
\end{align*}
Therefore,
\begin{align*}
    \delta_m^{(k)}&\le \bar\Delta^{(k)}/N^{(k)}= n^{-(1/4)(1-(3/5)^{k})}(\log n)^{-3\epsilon}\\
    \delta_z^{(k)}&\le \bar\Delta^{(k)}_z/N^{(k)}_z=n^{-(1/4)(1-(3/5)^{k})}(\log n)^{-3\epsilon}
\end{align*}
The choice of parameters guarantees that $p_m^{\ast}\in [\underline p_{m}^{(k)},\overline p_{m}^{(k)}]$ and $z^{\ast}\in [\underline{z}^{(k)},\overline z^{(k)}]$
occur with high probability.
Moreover, at the beginning of phase $K$, the precision of $[\underline p_{m}^{(K)},\overline p_{m}^{(K)}]$ and $[\underline{z}^{(K)},\overline z^{(K)}]$
is $\bar\Delta^{(K)}\sim \bar\Delta^{(K)}_z\sim n^{-1/4}$.

Next we point out the connections to the algorithm in \citet{wang2014close}.
When the capacity is sufficient, then this algorithm is closely related to that in \citet{wang2014close}: both target the precision $n^{-1/4}$ of the interval estimators for $p_m^{\ast}$, and our problem becomes $M$ independent learning problems in \citet{wang2014close}.
This is why the choices of $\tau^{(k)}$, $ N^{(k)}$, $\bar\Delta^{(k)}$ and $K$ are almost identical to that of \citet{wang2014close} except for logarithmic terms.
The design and analysis diverge for insufficient capacity.
In order to track the dual variable, we construct interval estimators for $z$, which is not needed in \citet{wang2014close}.
In the last phase, based on the estimation of the dual variable, we solve for the optimal prices of all the types simultaneously.
This is how we overcome the curse of dimensionality.
Note that the dimensionality issue doesn't arise in \citet{wang2014close}.

\section{Analysis}\label{sec:analysis}
In this section, we analyze the regret of the primal-dual learning algorithm.
To simplify the notation, we sometimes resort to a less rigorous expression, such as $\PR(A)=1-O(n^{-1})$;
its equivalence to $\PR(A^c)=O(n^{-1})$ should be clear in the context.

Before proceeding, we introduce a modified stochastic system.
Technically, if the inventory is depleted at $t$, then $P_m(t)$ must be switched to $p_{\infty}$, a choke price at which the demand of type-$m$ consumers is turned off, for all $m$.
We use a similar simplification to \citet{lei2014near} and
consider a slightly different problem.
When the inventory is depleted, instead of forced to set $p_{\infty}$ for all types of consumers, the firm can still use prices between $[\underline p, \overline p]$.
To accommodate the extra demand, it outsources the extra demand at a unit cost $\overline p$.
Denote the expected revenue of this modified system by $\tilde J^{\pi}_n$.
Because $\overline p$ is higher than the price charged, we must have $\tilde J_n^{\pi}\le J_n^{\pi}$.
To bound $J_n^D-J_n^{\pi}$, it suffices to bound $J_n^D-\tilde{J}_n^{\pi}$.
Therefore, from now on, we investigate the pricing policy $P_m(t)$ associated with the algorithm without switching to $p_{\infty}$ once the inventory is depleted.

\begin{remark}\label{rmk:no-constraint}
The benefit of studying $\tilde{J}_n^{\pi}$ instead of $J_n^{\pi}$ is that the pricing policy $\pi$ can be implemented for $t\in[0,T]$ without having to switch to $p_{\infty}$ at the stopping time at which the inventory is depleted.
This simplifies the analysis.
\end{remark}

We first show that the number of phases is growing slowly in $n$.
\begin{lemma}\label{lem:K}
    For $n\ge 3$, the total number of phases $K\le 3\log n+3$.
\end{lemma}
We next show that the last phase takes the majority of the season.
\begin{lemma}\label{lem:total-length-exploration}
    The total length of phases prior to phase $K$, $\sum_{k=1}^{K-1} \tau^{(k)}$ is less than or equal to $T/2$ for $n\ge \exp((8/T)^{1/\epsilon})$.
\end{lemma}
Consider the following events which are measurable with respect to $\mathcal F_{t_k}$:
\begin{align*}
    A_{k} &= \cap_{m=1}^M\left\{ p_m^{\ast}\in [\underline p_{m}^{(k)},\overline p_{m}^{(k)}]\right\}\\
    B_k &= \left\{z^{\ast}\in [\underline{z}^{(k)},\overline z^{(k)}]\right\}\\
    C_k & = \cap_{m=1}^M\left\{\mathcal P_m(z)\in [\underline p_{m}^{(k)},\overline p_{m}^{(k)}]\; \forall z\in [\underline{z}^{(k)},\overline z^{(k)}]\right\}
\end{align*}
By design, $A_k$ and $B_k$ are the key to the success of the algorithm.
If in some phase $k$, the interval estimators $[\underline p_{m}^{(k)},\overline p_{m}^{(k)}]$ and $[\underline{z}^{(k)},\overline z^{(k)}]$ do not contain $p_m^{\ast}$ and $z^{\ast}$,
then \eqref{eq:dual_algorithm} and \eqref{eq:primal_algorithm} do not make sense any more.
To make things worse, the optimal primal/dual pair $(p_m^{\ast},z^{\ast})$ cannot be recovered in subsequent phases and the learning policy is doomed to fail.
Therefore, we want to show that $A_k\cap B_k$ occurs with high probability.
The event $C_k$ is also crucial.
Note that to estimate $z^{\ast}$, the algorithm solves a discrete and empirical version of the dual function, i.e., \eqref{eq:dual_algorithm}.
If $\mathcal P_m(z)$ does not fall into the interval estimator for some $z$,
then $\max_{j_m = 0,\ldots, N^{(k)}} \hat d_{m,j_m}^{(k)}(p_{m,j_m}^{(k)}-z)$ in \eqref{eq:dual_algorithm} may provide a negatively biased estimator for $\mathcal R_m(z)$.
As a result, the minimization in \eqref{eq:dual_algorithm} may not find the correct value.

From the definitions, it is easy to see that $A_k\supseteq B_k\cap C_k$.
The following two lemmas show that $B_k$ and $C_k$ occur with high probability.
%Therefore, it is sufficient to show $\PR(\cap_{i=0}^K (B_i\cap C_i))=1-O(1/n)$.
\begin{lemma}\label{lem:Bi}
    For $k=1,\ldots,K-1$, conditional on $B_k\cap C_k$, $\PR(B_{k+1}|B_k\cap C_k)=1-O(1/n)$.
\end{lemma}

\begin{lemma}\label{lem:Ci}
    Conditional on $B_k\cap C_k\cap B_{k+1}$, $\PR(C_{k+1}|B_k\cap C_k\cap B_{k+1})=1-O(1/n)$.
\end{lemma}
Combining Lemma~\ref{lem:Bi} and Lemma~\ref{lem:Ci}, we have the following lemma, which states that at the beginning of phase $K$, the probability that the algorithm ``goes wrong'' is negligible.
\begin{lemma}\label{lem:bi-ci}
    \begin{equation*}
        \PR\left(\cap_{k=1}^K\left\{A_k\cap B_k\cap C_k\right\}\right) = 1-O((\log n)^2 n^{-1}).
    \end{equation*}
\end{lemma}

The following two lemmas characterize the cumulative sales $S(t)$ at $t=t_K$, the beginning of phase $K$.
Recall that the cumulative sales $S(t_K)$ can be expressed as $\int_0^{t_K}\sum_{m=1}^MdN_{m,t}(nd_m(P_m(t)))$.
Therefore, $\E[S(t_K)] =n\E[\int_0^{t_K}\sum_{m=1}^Md_m(P_m(t))dt]$.
Also note that in the fluid model, the inventory level at $t_K$ is $n\sum_{k=1}^{K-1}\tau^{(k)}\sum_{m=1}^Md_m(p_m^{\ast})$.
Therefore, the next two lemma state that the cumulative sales process does not deviate too much from that in the fluid system at the beginning of phase $K$.
\begin{lemma}\label{lem:inventory_tK1}
    At the beginning of phase $K$, the conditional expectation of $S(t_K)$ given the pricing policy $P_m(t)$ satisfies
    \begin{align*}
        \PR\left(\left| \int_0^{t_K}\sum_{m=1}^Md_m(P_m(t))dt-\sum_{k=1}^{K-1}\tau^{(k)}\sum_{m=1}^Md_m(p_m^{\ast}) \right| >n^{-1/4}(\log n)^{1+8\epsilon} \right) = O((\log n)^2n^{-1}).
    \end{align*}
\end{lemma}
\begin{lemma}\label{lem:inventory_tK2}
    At the beginning of phase $K$, $S(t_K)$ satisfies
    \begin{align*}
        \PR\left(\left| S(t_K)-n\sum_{k=1}^{K-1}\tau^{(k)}\sum_{m=1}^Md_m(p_m^{\ast}) \right| >2n^{3/4}(\log n)^{1+8\epsilon} \right) = O((\log n)^{-2}n^{-1/2}).
    \end{align*}
\end{lemma}
Roughly speaking (ignoring the logarithmic terms), Lemma~\ref{lem:inventory_tK1} and \ref{lem:inventory_tK2} show that the inventory level at the beginning of phase $K$ misses the target inventory level in the fluid system by $n^{3/4}$.
This is consistent with the precision of the price interval at $t_K$, which satisfies $\bar\Delta^{(K)}\sim n^{-1/4}$.
%Therefore, the task of phase $K$ is to further correct the error of magnitude $n^{-1/4}$, in order to finally meet the target regret $n^{-1/2}$.

\subsection{Sufficient Capacity}
We next bound the regret when $z^{\ast}= 0$, i.e., when the capacity is not constrained.
\begin{proposition}\label{prop:regret-z<=0}
    When $z^{\ast}\le 0$, we have $J_n^D-\E[\tilde{J}_n^{\pi}]=O((\log n)^{2+16\epsilon}n^{1/2})$.
    Therefore,
    \begin{equation*}
        R_n^{\pi}(T,c)=O((\log n)^{2+16\epsilon}n^{-1/2}).
    \end{equation*}
\end{proposition}
The major steps of the proof are sketched below. We first express $\E[\tilde{J}_n^{\pi}]$ as
\begin{equation}\label{eq:tildeJ_expression}
    n\sum_{k=1}^K\E\left[\int_{t_k}^{t_{k+1}} P_m(t)d_m(P_m(t))dt\right] - \overline p\E\left[ \left(\sum_{m=1}^M\int_0^Td N_{m,t}(nd_m(P_m(t))) -nc\right)^+\right].
\end{equation}
The first term is the expected revenue generated in each phase, and the second term accounts for the outsourcing cost explained in Remark~\ref{rmk:no-constraint}.
Moreover, $J_n^D$ can be expressed as $\sum_{k=1}^K\E\left[\int_{t_k}^{t_{k+1}} p_m^{\ast}d_m(p_m^{\ast})dt\right]$.
Thus, the difference between $J_n^{D}$ and the first term of \eqref{eq:tildeJ_expression} can thus be bounded by
\begin{align}\label{eq:regret_each_phase}
    &n\sum_{k=1}^K\E\left[\int_{t_k}^{t_{k+1}} \left(p_m^{\ast}d_m(p_m^{\ast})-P_m(t)d_m(P_m(t))\right)dt\right]\notag\\
    \sim & n\sum_{k=1}^K\E\left[\int_{t_k}^{t_{k+1}} \left(d_m(p_m^{\ast})-d_m(P_m(t))\right)^2dt\right]\sim n\sum_{k=1}^K\E\left[\int_{t_k}^{t_{k+1}} \left(p_m^{\ast}-P_m(t)\right)^2dt\right]\notag\\
    \sim & n\sum_{k=1}^K\tau^{(k)}(\bar\Delta^{(K)})^2.
\end{align}
Because $p^{\ast}_m$ is the unconstrained maximizer of $pd_m(p)$, we can apply a quadratic bound in the second line by Assumption~\ref{asp:convexity}.
The last line follows from the high-probability event $A_k$, which implies that $|P_m(t)-p^{\ast}_m|\le \bar\Delta^{(k)}$ in phase $k\le K-1$.
At the beginning of phase $K$, we can show that Step~\ref{step:sufficient_capacity} is triggered with high probability.
In this case, $|P_m(t)-p^{\ast}_m|=|p_m^{(K)}-p_m^{\ast}|=|\overline p^{(K)}_m-p_m^{\ast}+\alpha|\le \bar\Delta^{(K)}+\alpha$ is not necessarily bounded by $\bar\Delta^{(K)}$.
However, $\alpha$ is chosen carefully to match the order of $\bar\Delta^{(K)}$.
By the choice of the parameters, the order of $n\sum_{k=1}^K\tau^{(k)}(\bar\Delta^{(K)})^2$ is $O((\log n)^{2+16\epsilon}n^{1/2})$.

To bound the second term of \eqref{eq:tildeJ_expression}, we show that the mean of the total sales over the horizon, $n\E\left[ \sum_{m=1}^M\int_0^Td_m(P_m(t))dt \right]$, does not exceed $nc$.
This is achieved by charging a markup $p_m^{(K)}=\overline p^{(K)}_m+\alpha$ in phase $K$ and this is exactly the purpose of introducing $\alpha$.
Since the random fluctuation of Poisson arrivals is bounded by $O(n^{1/2})$,
we obtain the regret stated in Proposition~\ref{prop:regret-z<=0}.

\subsection{Insufficient Capacity}
When $z^{\ast}>0$, the inventory is depleted at $T$ in the fluid system.
We first show that the cumulative sales at the end of the horizon under the algorithm misses the target $nc$ by $O(n^{1/2})$.
\begin{lemma}\label{lem:inventory-XT}
    When $z^{\ast}>0$, we have
    \begin{equation}\label{eq:inventory-deviation}
        \E[|S(T)-nc|] = O((\log n)^{\epsilon}n^{1/2})
    \end{equation}
    and
    \begin{equation}\label{eq:rate-deviation}
        \E\left[\left|\int_{0}^T\sum_{m=1}^Md_m(P_m(t))dt-c\right|\right] = O((\log n)^{\epsilon}n^{-1/2}).
    \end{equation}
\end{lemma}
The intuition behind the proof is explained below.
We first show that $D_l^{(K)}$ and $D_u^{(K)}$ estimate $n\sum_{m=1}^Md_m(p_m^l)$ and $n\sum_{m=1}^Md_m(p_m^u)$
with precision $n^{1/2}$ (ignoring the logarithmic terms),
because they are Poisson random variables with means of order $n$ and thus the standard deviations are $O(n^{1/2})$.
With such precision, $\theta$ (Step~\ref{step:theta}) approximately solves
\begin{equation*}
    n(T-t_K)\sum_{m=1}^M (\theta d_m(p_m^l)+(1-\theta) d_m(p_m^u)) \approx nc - S(t_K).
\end{equation*}
Note that regardless of the first $K-1$ phases, the remaining inventory is $nc-S(t_K)$ at $t_K$.
Using $p_m^l$ ($p_m^u$) for a fraction $\theta$ ($1-\theta$) of phase $K$ serves as a corrective force to
ensure the aggregate sales over the horizon to be close to $nc$ (the error bound $(\log n)^{\epsilon}n^{1/2}$ is caused by the random fluctuation of the Poisson arrivals).

Algorithm~\ref{alg:primal-dual} does not explore the price space in phase $K$, and thus the precision of $p_m^l$ and $p_m^u$ in Step~\ref{step:pml_pmu} is of order $n^{-1/4}$.
Therefore, one would expect the sales to type-$m$ consumers
miss the target $nTd_m(p^{\ast}_m)$ by $n^{-1/4}\times n=n^{3/4}$.
This is indeed the case.
However, Lemma~\ref{lem:inventory-XT} guarantees that a simple exploration (Step~\ref{step:p_ml}, \ref{step:p_mu} and \ref{step:theta}) is effective and leads to higher precision ($O(n^{1/2}$) for the aggregate sales of all types of consumers.
This is crucial to the proof of the next proposition.

%\nc{need a lemma showing that when $z^{\ast}=0$, the algo enters case 1 with high prob; or just argue $A_K$ is a superset of that}
\begin{proposition}\label{prop:regret-z>0}
    When $z^{\ast}>0$, we have $J_n^D-\E[\tilde{J}_n^{\pi}]=O((\log n)^{2+18\epsilon}n^{1/2})$. Therefore,
    \begin{equation*}
        R_n^{\pi}(T,c)=O((\log n)^{2+18\epsilon}n^{-1/2}).
    \end{equation*}
\end{proposition}
Different from Proposition~\ref{prop:regret-z<=0}, $p^{\ast}_m$ is not the unconstrained maximizer of $p_m^{\ast}d_m(p_m^{\ast})$ when $z^{\ast}>0$.
Therefore, if we follow \eqref{eq:regret_each_phase}, the difference is approximately
\begin{align}\label{eq:z>0_wrong_regret}
    &n\sum_{k=1}^K\E\left[\int_{t_k}^{t_{k+1}} \left(p_m^{\ast}d_m(p_m^{\ast})-P_m(t)d_m(P_m(t))\right)dt\right]\notag\\
    \sim & n\sum_{k=1}^K\E\left[\int_{t_k}^{t_{k+1}} \left|d_m(p_m^{\ast})-d_m(P_m(t))\right|dt\right]\sim n\sum_{k=1}^K\tau^{(k)}\bar\Delta^{(K)}\sim n^{3/4},
\end{align}
which clearly does not meet our target $n^{1/2}$.
The remedy to this situation is the following key observation.
Let $r^{\ast\prime}_m$ and $r^{\ast\prime\prime}_m$ be the first- and second-order derivative of $r_m(\lambda)=\lambda d^{-1}_m(\lambda)$ at $\lambda=d_m(p_m^{\ast})$.
By Taylor's expansion, the difference in revenue rate can be expressed as
\begin{align*}
    &p_m^{\ast}d_m(p_m^{\ast})-P_m(t)d_m(P_m(t)) \\
    \le& r_m^{\ast\prime}(d_m(p_m^{\ast})-d_m(P_m(t)))+|r^{\ast\prime\prime}_m|(d_m(p_m^{\ast})-d_m(P_m(t)))^2.
\end{align*}
Because $\lambda = d_m(p_m^{\ast})$ maximizes $r_m(\lambda)-\lambda z^{\ast}$ by the primal-dual formulation (Proposition~\ref{prop:dual}), the first-order condition implies $r^{\ast\prime}_1=r^{\ast\prime}_2=\ldots=r^{\ast\prime}_M=z^{\ast}$.
Therefore, an improved bound for \eqref{eq:z>0_wrong_regret} can be derived:
\begin{align*}
    &n\sum_{k=1}^K\E\left[\int_{t_k}^{t_{k+1}} \left(p_m^{\ast}d_m(p_m^{\ast})-P_m(t)d_m(P_m(t))\right)dt\right]\\
    \sim  &nz^{\ast}\E\left[\int_{0}^{T}\sum_{m=1}^M( d_m(p_m^{\ast})-d_m(P_m(t)))dt\right]+ n\E\left[\int_{0}^{T}\sum_{m=1}^M|r^{\ast\prime\prime}_m|( d_m(p_m^{\ast})-d_m(P_m(t)))^2dt\right]\\
    \sim  &nz^{\ast}\E\left[c-\int_{0}^{T}\sum_{m=1}^M d_m(P_m(t))dt\right]+ M_4n\E\left[\int_{0}^{T}\sum_{m=1}^M( d_m(p_m^{\ast})-d_m(P_m(t)))^2dt\right].
\end{align*}
The first term can be bounded by $(\log n)^{\epsilon}n^{1/2}$ by Lemma~\ref{lem:inventory-XT}; this is the reason why we need to bound the aggregate sales.
The second term is of the same order as $n\sum_{k=1}^K\tau^{(k)}(\bar\Delta^{(K)})^2$ with high probability,
which has been shown to meet the target in the remarks following Proposition~\ref{prop:regret-z<=0}.

Combining Proposition~\ref{prop:regret-z<=0} and Proposition~\ref{prop:regret-z>0} and recalling that $\epsilon$ can be an arbitrarily small constant, we obtain the main theorem.
\begin{theorem}\label{thm:main}
    Suppose Assumptions~\ref{asp:convexity} and \ref{asp:initial} hold. For any $\delta>0$, we can select $\epsilon$ such that the regret of the primal-dual learning algorithm satisfies
    \begin{equation*}
        R_n^{\pi}(T,c)\le C(\log n)^{2+\delta}n^{-1/2},
    \end{equation*}
    where the constant $C$ is independent of $n$.
\end{theorem}
By \citet{besbes2009dynamic}, no learning policy can achieve regret that grows slower than $n^{-1/2}$ with $M=1$.
Therefore, ignoring the logarithmic terms, the primal-dual learning algorithm achieves near-optimal regret.

\section{Discussion}\label{sec:discuss}
In this section, we discuss several salient features of the primal-dual learning algorithm and clarify our findings in comparison to the literature.
\subsection{The Fundamental Limit of Online Learning}\label{sec:contribution-discussion}
This paper sets out on a quest to identify the fundamental limit of online learning (the optimal rate of regret) in the multi-product dynamic pricing problem over a network with a single resource.
The first benchmark is provided in \citet{slivkins2014contextual}: for a general Lipschitz-continuous objective function with a $d$-dimensional decision vector without resource constraints, the optimal rate of regret is $n^{-1/(2+d)}$.
Recently, \citet{chen2018nonparametric} show that when the objective function is locally concave, then the optimal rate is slightly better, and may be adjusted to $n^{-2/(3+d)}$.
The number of decisions $d$ is the number of products in network dynamic pricing, or $M$ in this paper.
Note that setting $d=1$ recovers the rate $n^{-1/2}$ in dynamic pricing (one product with one constraint) \citep{besbes2009dynamic,wang2014close}.
This is not good news: as the number of products $d$ increases, the rate deteriorates dramatically, which is referred to as the curse of dimensionality.
It is also consistent with the finding in \citet{besbes2012blind}.
The quest to understand whether online learning is complicated by dimensionality in  multi-product dynamic pricing problem is for continuous decision variables as is the case in dynamic pricing.  The rate of regret from papers assuming discrete decisions \citep{badanidiyuru2013bandits,agrawal2014bandits,ferreira2017online} does not apply to our setting.

Some recent papers have shed new light on this fundamental problem.
\citet{chen2019network} design a learning algorithm for network dynamic pricing that achieves rate $n^{-1/5}$
regardless of the number of products or constraints.
Although the rate doesn't seem to be optimal, it does not depend on the dimension of the problem (the number of products or the number of resources).
One may then question the fundamental difference between a general learning problem \citep{slivkins2014contextual} and network dynamic pricing what property makes the latter easier to learn?
\citet{li2019dimension} and references there in reveal that it may be due to the intrinsic concavity structure of network dynamic pricing.
In particular, \citet{li2019dimension} show that in the setting of \citet{slivkins2014contextual}, if the objective function is concave, then the regret is dimension-free ($n^{-1/2}$), because one can resort to gradient-based algorithms.
In network dynamic pricing, Assumption~\ref{asp:convexity} is commonly adopted to guarantee good behavior of the optimal policy.
Since the objective function is concave after one transforms the decision variable from price to quantity,
it somehow inherits the dimension-free nature.

This opens up a new research question: can the  $n^{-1/2}$ rate be obtained for learning in network dynamic pricing?
This paper takes a step forward to a positive answer of this question as we can confirm the rate $n^{-1/2}$ under two simplifications: The objective function in our setting can be expressed as the sum of the revenues collected from each of the $M$ types of consumers, and thus \emph{separable} in terms of the decision variables (prices); there is a single resource constraint.
The first simplification implies that, when the inventory constraint is not binding, then the optimal rate of regret should be the same as $M$ independent learning problems with one decision variable, i.e., $O(Mn^{-1/2})$.
It doesn't lead to dimension-free regret automatically when the constraint is binding, as the $M$-dimensional decision variables can be regarded as a vector on a $(M-1)$-dimensional manifold, after a proper transformation.
On the manifold, the objective function is not separable in terms of the transformed decision variables any more.
The second simplification can be relaxed as argued in Section~\ref{sec:multiple-resources}.
Therefore, our result confirms that $n^{-1/2}$ may be achievable.
Moreover, we provide the first dimension-free algorithm that is not gradient-based in the context.

%The second benchmark is \citet{chen2018self}, who establish the rate $n^{-1/2}$ when the objective function is infinitely differentiable with bounded derivatives.
%This may seem to have solved the problem with a stronger assumption.
%However, in nonparametric statistics, the fundamental limit of estimation (closely related to the rate of regret) is known to depend on the degree of smoothness, e.g., Sobolev classes and  H\"older classes \citep[Definition 1.2 and 1.4 of][]{tsybakov2009introduction}.
%Setting an infinite degree of smoothness does not clarify the mystery under general assumptions.
%In particular, in a generic learning problem in \citet{slivkins2014contextual}, the optimal rate of regret does not depend on $d$ as well if the function is infinitely differentiable with bounded derivatives.

%The result in this paper provides a starting point for the understanding of the fundamental question.
%We find that for one resource constraint, the optimal rate of regret is $\sqrt{T}$, regardless of the dimension of the decision variables.
%If this is true for multiple resource constraints as well, then learning in network revenue management is
%fundamentally easier than a general objective function \citep{slivkins2014contextual}.
%This news would be exciting for academics in the revenue management community as well as industry practitioners.

\subsection{Learning in the Primal and Dual Spaces}
The algorithm provides point and interval estimators for \emph{both} the primal and dual optimal solutions, $p_m^{\ast}$ and $z^{\ast}$, in each phase.
%To the best of our knowledge, this is the first algorithm in the literature that conducts learning in the dual space\footnote{\citet{badanidiyuru2013bandits} is an exception. However, the algorithm in that paper is not actively learning the dual \emph{optimal} solution.}.
It turns out that such learning is necessary, as the regret of a policy that only learns the primal space incurs much higher regret (see, e.g., Section 4 of \citealt{besbes2012blind}).
This is not surprising given the primal-dual formulation in Section~\ref{sec:full_info} since the revenues collected from different types of consumers are only coupled through the inventory constraint. Therefore, having an accurate estimator for the dual optimal solution $z^{\ast}$ helps to decouple the problem into $M$ independent learning problems of the form $\max_{p} d_m(p)(p-z^{\ast})$ for $m=1,\ldots,M$.
These independent learning problems are known to have regret $O(n^{-1/2})$ (the parametric version of such problems is solved in \citealt{keskin2014dynamic,denboer2014simul}).
However, $z^{\ast}$ is not given upfront and has to be learned.
The key design of the algorithm is to nest the learning processes in the primal and dual spaces to narrow down the primal/dual optimal solutions sequentially.

\subsection{Multiple Resource Constraints}\label{sec:multiple-resources}
The motivation of the study is personalized pricing, which can be recast as a multi-product dynamic pricing problem with a single constrained resource as mentioned in the introduction.
We believe our algorithm may be applied to a generic multi-product dynamic pricing problemwith multiple resource constraints, as long as the demand is separable, as mentioned in Section~\ref{sec:contribution-discussion}.
Next we briefly introduce the extension to multiple resources.

Suppose there are $L$ resources, with initial capacity $\bm c=(c_1,\dots,c_L)$, and product $m$ consumes $a_{ml}$ units of resource $l$.
The Lagrangian \eqref{eq:lagrangian} can be reformulated as
\begin{equation*}
    L(\bm p(t), \bm z) = \sum_{l=1}^L c_l z_l +\sum_{m=1}^M\int_{t=0}^T (p_m(t)-\sum_{l=1}^L a_{ml}z_l) d_m(p_m(t)) dt.
\end{equation*}
For a fixed $\bm z$  we define
\begin{align*}
    \mathcal R_m(\bm z)&\triangleq \max_{p}\left\{ d_m(p)(p-\sum_{l=1}^L a_{ml}z_l)\right\}\\
    \mathcal P_m(\bm z)& \triangleq \argmax_{p}\left\{ d_m(p)(p-\sum_{l=1}^L a_{ml}z_l)\right\}.
\end{align*}
and the dual function
\begin{equation*}
    g(\bm z) = \max_{\bm p(t)}\left\{L(\bm p(t), \bm z)\right\}.
\end{equation*}

As for the implementation of the algorithm, the exploration of the first $K-1$ phases is essentially the same (Step~\ref{step:for-loop-begin} to Step~\ref{step:finish-test-price}).
When finding the point estimator for $\bm z^{(k)\ast}$, Equation \eqref{eq:dual_algorithm} is vectorized, and the optimal vector $\bm z^{(k)\ast}$ is obtained by testing the grid in the confidence set of $\bm z$, which is a hyper-rectangle since each entry of $\bm z$ has a confidence interval.
Normally the grid size explodes exponentially in the size of $\bm z$, i.e., $L$, and this is precisely causing the curse of dimensionality.
However, in this algorithm \eqref{eq:dual_algorithm} is solved offline, i.e., the computation is not counted toward the final regret.
In the beginning of the last phase, the dual/primal variables may have been estimated with error $n^{-1/4}$, similar to \eqref{eq:one-four-convergence}, under properly chosen parameters.
For products with sufficient capacity (a product $m$ enjoys sufficient capacity if $0\in [\underline z^{(K)}_{l}, \bar z^{(K)}_{l}]$ for all $l$ such that $a_{ml}>0$), this already leads to the optimal regret.
For products with insufficient capacity, it is unclear how the linear interpolation \eqref{eq:theta} can be implemented in the high dimension.
We thus leave the algorithm and regret analysis in this case for future research.
%we select the resources whose corresponding $z_l$ is likely to be zero (according to the confidence intervals after phase $K-1$, as in Step~\ref{step:sufficient_capacity} and~\ref{step:insufficient_capacity}) and the algorithm similarly has two different treatments.
%Note that in this case, a product $m$ enjoys sufficient capacity if $0\in [\underline z^{(K)}_{l}, \bar z^{(K)}_{l}]$ for all $l$ such that $a_{ml}>0$.
%This guarantees that $\mathcal P_m(\bm z)$ equals to the unconstrained maximizer and quadratic convergence.
%Otherwise the product has insufficient capacity.

\subsection{Controlling the Aggregate Sales}
As explained in Section~\ref{sec:description} and the remark following Proposition~\ref{prop:regret-z>0},
when $z^{\ast}>0$, the interval estimators for $p_m^{\ast}$ at the beginning of phase $K$ are still too wide ($n^{-1/4}$).
Instead of exploring the price space further and attempting to narrow down the intervals in phase $K$,
the algorithm simply controls the aggregate sales within an error margin of $n^{1/2}$ (Step~\ref{step:insufficient_capacity} to \ref{step:p_mu_theta} in Algorithm~\ref{alg:primal-dual} and Lemma~\ref{lem:inventory-XT}).
Focusing on a single quantity (the aggregate sales) turns out much easier than controlling $M$ decision variables.
In fact, we suspect that no learning policy could estimate all $p_m^{\ast}$ with precision $n^{-1/2}$ at the end of the horizon\footnote{The case we study is different from \citet{wang2014close}, in which the market-clearing price $d^{-1}(c/T)$ can be learned with precision $n^{-1/2}$. In our case, there are $m$ types of consumers and there are still $m-1$ degrees of freedom when a vector of prices $(p_1,\ldots,p_m)$ are market clearing.}.

Surprisingly, controlling the aggregate sales is sufficient to meet the target regret, even though the estimators for the optimal prices are not precise enough.
The reason is explained by the remarks following Proposition~\ref{prop:regret-z>0}.
In particular, the derivatives of the revenue rates $\lambda d^{-1}_m(\lambda)$ with respect to the demand rate $\lambda$ are all equal to $z^{\ast}$ at optimality $\lambda=d_m(p^{\ast}_m)$ for $m=1,\ldots,M$.
This allows the firm to control the regret by the deviation of the aggregate sales.

%\subsection{Overcoming the Curse of Dimensionality}
%As shown by \citet{besbes2012blind,slivkins2014contextual}, the dimension of the decision vector ($M$ in this case)
%usually significantly complicates learning.
%\citet{slivkins2014contextual} shows that the regret for a generic learning problem (the objective is not necessarily concave) with continuum-armed bandits is at least $O(n^{-1/(2+d)})$ where $d$ is the dimension of the decision vector.
%In contrast, we are able to obtain a rate $n^{-1/2}$ whose exponent is independent of $M$, thus improving the regret $n^{-1/(M+3)}$ in \citet{besbes2012blind}, especially when $M$ is large.
%
%To explain this observation, note that in our problem the objectives $pd_m(p)$ for $m=1,\ldots,M$ are only coupled by the inventory constraint and otherwise independent.
%In other words, if the dual optimal solution $z^{\ast}$ is given upfront, then the learning problem is decomposed into $M$ independent ones, each with a single decision variable.
%Therefore, the effective dimension of the learning problem is one, equal to the dimension of the dual space.
%We believe this insight can be carried over to a general network revenue management problem \citep{besbes2012blind}, in which the dual space can be high-dimensional.
%In that case, the complexity of learning may be determined by the minimum of the dimensions of the primal and dual spaces.

\subsection{Data Reuse}
Across different phases, the interval estimators for the optimal price $p_m^{\ast}$ may overlap for some $m$.
In this case, the demand estimated at certain prices in the previous phases may be reused.
For example, if $[\underline p_m^{(k)},\overline p_m^{(k)}]$ and $[\underline p_m^{(k+1)},\overline p_m^{(k+1)}]$ overlap, then for some $j_1$ and $j_2$ the price grid points $p_{m,j_1}^{(k)}$ and $p_{m,j_2}^{(k+1)}$ may be close, and the demand estimate for $p_{m,j_1}^{(k)}$ can provide useful information for $p_{m,j_2}^{(k+1)}$.
In our algorithm, the data from previous phases are not reused mainly for the analysis, because data reuse introduces complex dependence between phases.
In practice, we believe that data reuse may facilitate the learning of the demand function and increase the efficiency of the policy.

\subsection{Discontinuous Demand}
In a recent paper, \cite{den2019discontinuous} study discontinuous demand functions, which is by far the most general assumption on smoothness,
although their demand function has a parametric form.
In our setting, Assumption~\ref{asp:convexity} rules out the possibility of discontinuity, and our algorithm is unlikely to work for discontinuous functions.
This is because discontinuous functions cannot be concave, and they break the primal-dual formulation and thus the foundation of the algorithm.
We also believe that in the nonparametric setting, no algorithm can achieve sublinear regret when the demand function can be discontinuous.
Consider the following example: there is one type of consumer ($M=1$); we normalize the time horizon ($T=1$) and there is sufficient capacity ($c=100$).
Let the price range be $p\in[0,1]$.
Consider the demand (revenue) function satisfying
\begin{equation*}
    pd(p) = \mu_k,\;p\in ((k-1)/K, k/K].
\end{equation*}
Because the capacity is not constrained, the problem is conceptually equivalent to the multi-armed bandit problem with $K$ arms.
(The difference of continuous/discrete time is non-consequential in this case.)
It is well-known that the minimax lower bound for the regret of such a learning problem is $O(\sqrt{Kn})$.
Since $K$ can be arbitrarily large, the regret cannot possibly be controlled.

%\subsection{Multiple Products}
%To accommodate horizontal differentiation, we can extend the network revenue management problem and include personalization.
%Since the problem now has three dimensions, the number of resources and products and the number of consumer types, it would be a fundamental and exciting direction to investigate how the regret depends on the three dimensions.

\bibliographystyle{chicago}
\bibliography{ref.bib}

\begin{thebibliography}{}

\bibitem[\protect\citeauthoryear{Agrawal and Devanur}{Agrawal and
  Devanur}{2014}]{agrawal2014bandits}
Agrawal, S. and N.~R. Devanur (2014).
\newblock Bandits with concave rewards and convex knapsacks.
\newblock In {\em Proceedings of the fifteenth ACM conference on Economics and
  computation}, pp.\  989--1006. ACM.

\bibitem[\protect\citeauthoryear{Araman and Caldentey}{Araman and
  Caldentey}{2009}]{araman2009dynamic}
Araman, V.~F. and R.~Caldentey (2009).
\newblock Dynamic pricing for nonperishable products with demand learning.
\newblock {\em Operations research\/}~{\em 57\/}(5), 1169--1188.

\bibitem[\protect\citeauthoryear{Auer, Ortner, and Szepesv{\'a}ri}{Auer
  et~al.}{2007}]{Auer2007}
Auer, P., R.~Ortner, and C.~Szepesv{\'a}ri (2007).
\newblock {\em Improved Rates for the Stochastic Continuum-Armed Bandit
  Problem}, pp.\  454--468.
\newblock Berlin, Heidelberg: Springer Berlin Heidelberg.

\bibitem[\protect\citeauthoryear{Badanidiyuru, Kleinberg, and
  Slivkins}{Badanidiyuru et~al.}{2013}]{badanidiyuru2013bandits}
Badanidiyuru, A., R.~Kleinberg, and A.~Slivkins (2013).
\newblock Bandits with knapsacks.
\newblock In {\em Foundations of Computer Science (FOCS), 2013 IEEE 54th Annual
  Symposium on}, pp.\  207--216. IEEE.

\bibitem[\protect\citeauthoryear{Badanidiyuru, Langford, and
  Slivkins}{Badanidiyuru et~al.}{2014}]{badanidiyuru2014resourceful}
Badanidiyuru, A., J.~Langford, and A.~Slivkins (2014).
\newblock Resourceful contextual bandits.
\newblock In {\em Conference on Learning Theory}, pp.\  1109--1134.

\bibitem[\protect\citeauthoryear{Ban and Keskin}{Ban and
  Keskin}{2017}]{ban2017personalized}
Ban, G. and N.~B. Keskin (2017).
\newblock Personalized dynamic pricing with machine learning.
\newblock {\em Working paper\/}.

\bibitem[\protect\citeauthoryear{Besbes and Zeevi}{Besbes and
  Zeevi}{2009}]{besbes2009dynamic}
Besbes, O. and A.~Zeevi (2009).
\newblock Dynamic pricing without knowing the demand function: Risk bounds and
  near-optimal algorithms.
\newblock {\em Operations Research\/}~{\em 57\/}(6), 1407--1420.

\bibitem[\protect\citeauthoryear{Besbes and Zeevi}{Besbes and
  Zeevi}{2012}]{besbes2012blind}
Besbes, O. and A.~Zeevi (2012).
\newblock Blind network revenue management.
\newblock {\em Operations research\/}~{\em 60\/}(6), 1537--1550.

\bibitem[\protect\citeauthoryear{Broder and Rusmevichientong}{Broder and
  Rusmevichientong}{2012}]{broder2012dynamic}
Broder, J. and P.~Rusmevichientong (2012).
\newblock Dynamic pricing under a general parametric choice model.
\newblock {\em Operations Research\/}~{\em 60\/}(4), 965--980.

\bibitem[\protect\citeauthoryear{Bubeck and Cesa-Bianchi}{Bubeck and
  Cesa-Bianchi}{2012}]{bubeck2012regret}
Bubeck, S. and N.~Cesa-Bianchi (2012).
\newblock Regret analysis of stochastic and nonstochastic multi-armed bandit
  problems.
\newblock {\em Foundations and Trends{\textregistered} in Machine
  Learning\/}~{\em 5\/}(1), 1--122.

\bibitem[\protect\citeauthoryear{Bubeck, Munos, Stoltz, and
  Szepesv{\'a}ri}{Bubeck et~al.}{2011}]{bubeck2011x}
Bubeck, S., R.~Munos, G.~Stoltz, and C.~Szepesv{\'a}ri (2011).
\newblock X-armed bandits.
\newblock {\em Journal of Machine Learning Research\/}~{\em 12\/}(May),
  1655--1695.

\bibitem[\protect\citeauthoryear{Canonne}{Canonne}{2017}]{canonne}
Canonne, C. (2017).
\newblock A short note on poisson tail bounds.

\bibitem[\protect\citeauthoryear{Cesa-Bianchi and Lugosi}{Cesa-Bianchi and
  Lugosi}{2006}]{cesa2006prediction}
Cesa-Bianchi, N. and G.~Lugosi (2006).
\newblock {\em Prediction, learning, and games}.
\newblock Cambridge university press.

\bibitem[\protect\citeauthoryear{Chen and Gallego}{Chen and
  Gallego}{2019}]{chen2018nonparametric}
Chen, N. and G.~Gallego (2019).
\newblock Nonparametric pricing analytics with customer covariates.
\newblock {\em Working paper\/}.

\bibitem[\protect\citeauthoryear{Chen, Jasin, and Duenyas}{Chen
  et~al.}{2019}]{chen2018self}
Chen, Q., S.~Jasin, and I.~Duenyas (2019).
\newblock Nonparametric self-adjusting control for joint learning and
  optimization of multiproduct pricing with finite resource capacity.
\newblock {\em Mathematics of Operations Research\/}~{\em 44\/}(2), 601--631.

\bibitem[\protect\citeauthoryear{Chen and Shi}{Chen and
  Shi}{2019}]{chen2019network}
Chen, Y. and C.~Shi (2019).
\newblock Network revenue management with online inverse batch gradient descent
  method.
\newblock {\em Working paper\/}.

\bibitem[\protect\citeauthoryear{Cheung, Simchi-Levi, and Wang}{Cheung
  et~al.}{2017}]{cheung2017dynamic}
Cheung, W.~C., D.~Simchi-Levi, and H.~Wang (2017).
\newblock Dynamic pricing and demand learning with limited price
  experimentation.
\newblock {\em Operations Research\/}~{\em 65\/}(6), 1722--1731.

\bibitem[\protect\citeauthoryear{Cohen, Lobel, and Paes~Leme}{Cohen
  et~al.}{2016}]{cohen2016feature}
Cohen, M.~C., I.~Lobel, and R.~Paes~Leme (2016).
\newblock Feature-based dynamic pricing.
\newblock {\em Working paper\/}.

\bibitem[\protect\citeauthoryear{den Boer and Keskin}{den Boer and
  Keskin}{2020}]{den2019discontinuous}
den Boer, A. and N.~B. Keskin (2020).
\newblock Discontinuous demand functions: estimation and pricing.
\newblock {\em Management Science\/}~{\em Forthcoming}.

\bibitem[\protect\citeauthoryear{den Boer}{den Boer}{2015}]{den2015dynamic}
den Boer, A.~V. (2015).
\newblock Dynamic pricing and learning: historical origins, current research,
  and new directions.
\newblock {\em Surveys in operations research and management science\/}~{\em
  20\/}(1), 1--18.

\bibitem[\protect\citeauthoryear{den Boer and Zwart}{den Boer and
  Zwart}{2014}]{denboer2014simul}
den Boer, A.~V. and B.~Zwart (2014).
\newblock Simultaneously learning and optimizing using controlled variance
  pricing.
\newblock {\em Management Science\/}~{\em 60\/}(3), 770--783.

\bibitem[\protect\citeauthoryear{den Boer and Zwart}{den Boer and
  Zwart}{2015}]{den2015dynamicor}
den Boer, A.~V. and B.~Zwart (2015).
\newblock Dynamic pricing and learning with finite inventories.
\newblock {\em Operations research\/}~{\em 63\/}(4), 965--978.

\bibitem[\protect\citeauthoryear{Farias and Van~Roy}{Farias and
  Van~Roy}{2010}]{farias2010dynamic}
Farias, V.~F. and B.~Van~Roy (2010).
\newblock Dynamic pricing with a prior on market response.
\newblock {\em Operations Research\/}~{\em 58\/}(1), 16--29.

\bibitem[\protect\citeauthoryear{Ferreira, Simchi-Levi, and Wang}{Ferreira
  et~al.}{2018}]{ferreira2017online}
Ferreira, K.~J., D.~Simchi-Levi, and H.~Wang (2018).
\newblock Online network revenue management using thompson sampling.
\newblock {\em Operations research\/}~{\em 66\/}(6), 1586--1602.

\bibitem[\protect\citeauthoryear{Gallego and Topaloglu}{Gallego and
  Topaloglu}{2019}]{gallego2019revenue}
Gallego, G. and H.~Topaloglu (2019).
\newblock {\em Revenue management and pricing analytics}, Volume 209.
\newblock Springer.

\bibitem[\protect\citeauthoryear{Gallego, Topaloglu, et~al.}{Gallego
  et~al.}{2019}]{ggbook}
Gallego, G., H.~Topaloglu, et~al. (2019).
\newblock {\em Revenue management and pricing analytics}, Volume 209.
\newblock Springer.

\bibitem[\protect\citeauthoryear{Gallego and Van~Ryzin}{Gallego and
  Van~Ryzin}{1997}]{gallego1997multiproduct}
Gallego, G. and G.~Van~Ryzin (1997).
\newblock A multiproduct dynamic pricing problem and its applications to
  network yield management.
\newblock {\em Operations research\/}~{\em 45\/}(1), 24--41.

\bibitem[\protect\citeauthoryear{Javanmard and Nazerzadeh}{Javanmard and
  Nazerzadeh}{2016}]{javanmard2016dynamic}
Javanmard, A. and H.~Nazerzadeh (2016).
\newblock Dynamic pricing in high-dimensions.
\newblock {\em Working paper\/}.

\bibitem[\protect\citeauthoryear{Keskin and Birge}{Keskin and
  Birge}{2019}]{keskin2019dynamic}
Keskin, N.~B. and J.~R. Birge (2019).
\newblock Dynamic selling mechanisms for product differentiation and learning.
\newblock {\em Operations research\/}~{\em 67\/}(4), 1069--1089.

\bibitem[\protect\citeauthoryear{Keskin and Zeevi}{Keskin and
  Zeevi}{2014}]{keskin2014dynamic}
Keskin, N.~B. and A.~Zeevi (2014).
\newblock Dynamic pricing with an unknown demand model: Asymptotically optimal
  semi-myopic policies.
\newblock {\em Operations Research\/}~{\em 62\/}(5), 1142--1167.

\bibitem[\protect\citeauthoryear{Keskin and Zeevi}{Keskin and
  Zeevi}{2018}]{keskin2018incomplete}
Keskin, N.~B. and A.~Zeevi (2018).
\newblock On incomplete learning and certainty-equivalence control.
\newblock {\em Operations Research\/}~{\em 66\/}(4), 1136--1167.

\bibitem[\protect\citeauthoryear{Kleinberg, Slivkins, and Upfal}{Kleinberg
  et~al.}{2008}]{kleinberg2008multi}
Kleinberg, R., A.~Slivkins, and E.~Upfal (2008).
\newblock Multi-armed bandits in metric spaces.
\newblock In {\em Proceedings of the fortieth annual ACM symposium on Theory of
  computing}, pp.\  681--690. ACM.

\bibitem[\protect\citeauthoryear{Kleinberg}{Kleinberg}{2005}]{kleinberg2005nearly}
Kleinberg, R.~D. (2005).
\newblock Nearly tight bounds for the continuum-armed bandit problem.
\newblock In {\em Advances in Neural Information Processing Systems}, pp.\
  697--704.

\bibitem[\protect\citeauthoryear{Lei, Jasin, and Sinha}{Lei
  et~al.}{2017}]{lei2014near}
Lei, Y., S.~Jasin, and A.~Sinha (2017).
\newblock Near-optimal bisection search for nonparametric dynamic pricing with
  inventory constraint.
\newblock {\em Working paper\/}.

\bibitem[\protect\citeauthoryear{Li, Chen, and Hong}{Li
  et~al.}{2019}]{li2019dimension}
Li, W., N.~Chen, and L.~J. Hong (2019).
\newblock A dimension-free algorithm for contextual continuum-armed bandits.
\newblock {\em Working paper\/}.

\bibitem[\protect\citeauthoryear{Mahdavi, Yang, and Jin}{Mahdavi
  et~al.}{2013}]{mahdavi2013stochastic}
Mahdavi, M., T.~Yang, and R.~Jin (2013).
\newblock Stochastic convex optimization with multiple objectives.
\newblock In {\em Advances in Neural Information Processing Systems}, pp.\
  1115--1123.

\bibitem[\protect\citeauthoryear{Qiang and Bayati}{Qiang and
  Bayati}{2016}]{qiang2016dynamic}
Qiang, S. and M.~Bayati (2016).
\newblock Dynamic pricing with demand covariates.
\newblock {\em Working paper\/}.

\bibitem[\protect\citeauthoryear{Shalev-Shwartz et~al.}{Shalev-Shwartz
  et~al.}{2012}]{shalev2012online}
Shalev-Shwartz, S. et~al. (2012).
\newblock Online learning and online convex optimization.
\newblock {\em Foundations and Trends{\textregistered} in Machine
  Learning\/}~{\em 4\/}(2), 107--194.

\bibitem[\protect\citeauthoryear{Slivkins}{Slivkins}{2014}]{slivkins2014contextual}
Slivkins, A. (2014).
\newblock Contextual bandits with similarity information.
\newblock {\em The Journal of Machine Learning Research\/}~{\em 15\/}(1),
  2533--2568.

\bibitem[\protect\citeauthoryear{Wang, Deng, and Ye}{Wang
  et~al.}{2014}]{wang2014close}
Wang, Z., S.~Deng, and Y.~Ye (2014).
\newblock Close the gaps: A learning-while-doing algorithm for single-product
  revenue management problems.
\newblock {\em Operations Research\/}~{\em 62\/}(2), 318--331.

\end{thebibliography}

\singlespacing
\newpage
\begin{appendices}
    \section{Table of Notations}
\begin{table}[ht!]
    \centering
    \begin{tabular}{cl}
        \toprule
        Notation & Meaning\\
        \midrule
        $c$, $T$ & The initial inventory, the length of the selling season\\
        $n$ & The scaling index\\
        $r_m(\lambda)$ & $\lambda d_m^{-1}(\lambda)$\\
        $\epsilon$ & A small positive constant\\
        $K$ & The total number of phases\\
        $\tau^{(k)}$ & The length of phase $k$\\
        $t_k$ & The time of the beginning of phase $k$, $t_{K+1}=T$\\
        $P_m(t)$ & The (stochastic) pricing policy for type-$m$ consumers at $t$\\
        $[\underline p,\overline p]$, $[0,\overline z]$ & The domain for the primal/dual variables\\
        $[\underline p_{m}^{(k)},\overline p_{m}^{(k)}]$, $[\underline{z}^{(k)},\overline z^{(k)}]$ & The intervals of primal/dual variables in phase $k$\\
        $\Delta_m^{(k)}$, $\Delta_z^{(k)}$& The length of $[\underline p_{m}^{(k)},\overline p_{m}^{(k)}]$ and $[\underline{z}^{(k)},\overline z^{(k)}]$ after truncation\\
        $\bar\Delta^{(k)}$, $\bar\Delta_z^{(k)}$& The length of $[\underline p_{m}^{(k)},\overline p_{m}^{(k)}]$ and $[\underline{z}^{(k)},\overline z^{(k)}]$ before truncation\\
        $N^{(k)}+1$, $N_z^{(k)}+1$ & The number of grid points in $[\underline p_{m}^{(k)},\overline p_{m}^{(k)}]$ and $[\underline{z}^{(k)},\overline z^{(k)}]$\\
        $\delta_m^{(k)}$, $\delta_z^{(k)}$ & The grid resolution, equal to $\Delta_m^{(k)}/N^{(k)}$ and $\Delta_z^{(k)}/N_z^{(k)}$\\
        $p_{m,j}^{(k)}$ & The $j$th price on the grid, $\underline p_{m}^{(k)}+j\delta_m^{(k)}$\\
        $D_{m,j}^{(k)}$ & The realized demand of type $m$ in phase $k$ for price $p_{m,j}^{(k)}$\\
        $\hat d^{(k)}_{m,j}$ & The empirical estimate for $d_m( p_{m,j}^{(k)})$\\
        $z_{j}^{(k)}$ & The $j$th dual variable on the grid, $\underline z^{(k)}+j\delta_z^{(k)}$\\
        $z^{(k)\ast}$ & The dual optimal solution on the primal/dual grids \\
                      & using the empirical estimates for $d_m(\cdot)$ \\
        $p_m^{(k)\ast}$ & The primal optimal solution on the price grids using the \\
                        & dual variable $z^{(k)\ast}$ and the empirical estimates for $d_m(\cdot)$\\
        $N_{m,t}( nd_m(P_m(t)))$ & The Poisson arrival process of type-$m$ consumers\\
        $A_{k}$ & The event $\cap_{m=1}^M\{ p_m^{\ast}\in [\underline p_{m}^{(k)},\overline p_{m}^{(k)}]\}$\\
        $B_k$ & The event $\{z^{\ast}\in [\underline{z}^{(k)},\overline z^{(k)}]\}$\\
        $C_k$ & The event $\cap_{m=1}^M\{\mathcal P_m(z)\in [\underline p_{m}^{(k)},\overline p_{m}^{(k)}]\; \forall z\in [\underline{z}^{(k)},\overline z^{(k)}]\}$\\
        \bottomrule
    \end{tabular}
    \caption{A summary of notations.}
    \label{tab:notation}
\end{table}
\section{Proofs}
We first provide two lemmas and their proofs that will be used repeatedly in the proof of the main result.
\begin{lemma}\label{lem:poisson}
    Suppose $r_n\ge \alpha n^{\beta}$ with $\beta>0$ and $\epsilon_n= \gamma(\log n)^{1/2+\eta}r_n^{-1/2}$ for $\eta>0$.
    For any $k>0$, we can find $C$ that is independent of $n$ (which may depend on $\alpha$, $\beta$, $\gamma$, $\eta$ and $k$) such that
    \begin{equation*}
        \PR(|N(r_n)-r_n|>r_n\epsilon_n)\le \frac{C}{n^{k}}
    \end{equation*}
    for all $n\ge 1$,
    where $N(\mu)$ is a Poisson random variable with mean $\mu$.
\end{lemma}
\begin{proof}[Proof of Lemma~\ref{lem:poisson}:]
    The result follows from the following bound for Poisson tail probabilities:
    Let $X$ be a Poisson random variable with mean $\lambda$. Then for any $x>0$, we have
    \begin{equation*}
        \PR(|X-\lambda|\ge x)\le 2e^{- \frac{x^2}{2(\lambda+x)}}.
    \end{equation*}
    For a proof, see, e.g., \citet{canonne}.
    To prove Lemma~\ref{lem:poisson}, letting $\lambda=r_n$ and $x=r_n\epsilon_n$ and we have
    \begin{align*}
        2e^{- \frac{x^2}{2(\lambda+x)}} = 2 \exp\left(- \frac{ \gamma^2(\log n)^{1+2\eta}}{2(1+\epsilon_n)}\right).
    \end{align*}
    Because $r_n$ grows polynomially in $n$, we have $\lim_{n\to\infty}\epsilon_n= 0$.
    Therefore, there exists a constant $C_1$ that only depends on $\alpha$, $\beta$, $\gamma$ and $\eta$, such that for $n\ge C_1$
    \begin{align*}
        \PR(|N(r_n)-r_n|>r_n\epsilon_n)&\le 2\exp(- \frac{ \gamma^2(\log n)^{1+2\eta}}{4})
        \le 2n^{-\gamma^2(\log n)^{2\eta}/4}.
    \end{align*}
    It is clear that for any $k>0$, the decay of the above term is faster than $n^{-k}$.
    Therefore, we can always find such a constant $C$.
\end{proof}
\begin{lemma}\label{lem:poisson-expect}
    Suppose $X$ is a Poisson random variable with mean $\lambda$, then we have
    \begin{align*}
        \E\left[|X-\lambda|\right]\le \sqrt{\lambda}.
    \end{align*}
\end{lemma}
\begin{proof}[Proof of Lemma~\ref{lem:poisson-expect}]
    Note that $\Var(X)=\lambda$. Applying the Cauchy-Schwartz inequality, we have
    \begin{align*}
        \E[|X-\lambda|]\le \E[(X-\lambda)^2]^{1/2}=\sqrt{\lambda}.
    \end{align*}
    This completes the proof.
\end{proof}

\begin{proof}[Proof of Proposition~\ref{prop:dual}:]
    Consider $\lambda_m (z)\triangleq \max_{\lambda}\left\{\lambda (d^{-1}_m(\lambda)-z)\right\}$.
    It is easy to see that $z_1>z_2\implies \lambda (d^{-1}_m(\lambda)-z_1)<\lambda (d^{-1}_m(\lambda)-z_2)\implies \lambda_m(z_1)<\lambda_m(z_2)$ so that $\lambda_m^{-1}(\cdot)$ is well-defined.
    Then $\mathcal P_m(z) = d_m^{-1}(\lambda_m(z))$.
    To show that $\mathcal P_m(z)$ is increasing in $z$ and $\mathcal P'_m(z)$ is bounded in part one,
    it suffices to show that $\lambda'_m(z)$ is bounded.
    The first-order condition for maximizing $\lambda (d^{-1}_m(\lambda)-z)$ yields
    \begin{align*}
        z= r'_m(\lambda)\;\implies \; (\lambda_m^{-1}(\lambda))' = r''_m(\lambda).
    \end{align*}
    Because $r_m(\cdot)$ is concave with a second-order derivative bounded above by $-M_3<0$ by Assumption~\ref{asp:convexity},
    $\lambda_m^{-1}$ is decreasing whose first-order derivative is bounded above by $[-M_4,-M_3]$.
    This implies that $\lambda_m(\cdot)$ is strictly decreasing in $z$ whose first-order derivative is bounded below by $-1/M_3$.
    Thus, we conclude that $\mathcal P_m(z)$ is strictly increasing and the derivative is bounded.

    Part two follows from the envelop theorem and the fact that $d_m(\cdot)$ is strictly decreasing by Assumption~\ref{asp:convexity} and the proof of part one.

    For part three, note that $g(z)=cz+T\sum_{m=1}^M \mathcal R_m(z)$.
    Since $\mathcal R_m(z)$ is twice differentiable and convex by part two, the claim follows.

    For part four, note that the objective function is concave and the constraint is linear in $d_m(p_m(t))$.
    Therefore, Slater's condition implies strong duality and complementary slackness.
    The proofs can also be found in \citet{ggbook}.
\end{proof}

\begin{proof}[Proof of Lemma~\ref{lem:K}:]
    From the expression of $\bar\Delta^{(k)}$ and $K$, we have
    \begin{align*}
        &(\bar\Delta^{(K)})^2=n^{-(1/2)(1-(3/5)^{K-1})}\le n^{-1/2}(\log n)^{2+16\epsilon}\\
        \implies&K\le \log_{5/3} \frac{\log n}{(4+32\epsilon)\log \log n}+2 \le  \log_{5/3} \frac{\log n}{10\log \log n}+\log_{5/3} (5/2)+2,
    \end{align*}
    because $\epsilon>0$.
    It has been shown in the proof of Lemma 2 in \citet{wang2014close} that
    \begin{equation*}
        \log_{5/3} \frac{\log n}{10\log \log n}+1\le 3\log n.
    \end{equation*}
    Therefore, $K\le 3\log n+3$
\end{proof}

\begin{proof}[Proof of Lemma~\ref{lem:total-length-exploration}:]
    By the choices of $K$ and $\bar\Delta^{(k)}$, we have $(\bar \Delta^{(K-1)})^2> n^{-1/2}(\log n)^{2+16\epsilon}$, and thus $n^{1/2\times(3/5)^{K-2}}\ge (\log n)^{2+16\epsilon}$.
    This implies that $\tau^{(K-1)}\le (\log n)^{-1-\epsilon}$.
    It is clear that $\tau^{(i)}$ is increasing in $i$.
    Thus by Lemma~\ref{lem:K}, $\sum_{i=1}^{K-1} \tau^{(i)}\le (K-1)(\log n)^{-1-\epsilon}\le 4(\log n)^{-\epsilon}\le T/2$ for $n\ge \exp((8/T)^{1/\epsilon})$.
\end{proof}

\begin{proof}[Proof of Lemma~\ref{lem:Bi}:]
    We will bound the probability of $B_{k+1}^c$, the complement of $B_{k+1}$.
    Note that the width of $[\underline{z}^{(k+1)},\overline z^{(k+1)}]$ is $\Delta_z^{(k+1)}$ and the middle point before truncation is $z^{(k)\ast}$.
    Therefore, when $B_{k+1}^c$ occurs, it implies that\footnote{If $[z^{(k)\ast}- \bar\Delta_z^{(k+1)}/2,z^{(k)\ast}+\bar\Delta^{(k+1)}_z/2]$ is not truncated by $[0,\bar z]$, then $\bar{\Delta}_z^{(k+1)}=\Delta_z^{(k+1)}$ and thus \eqref{eq:zstar-Delta} holds.
    Otherwise, because $z^{\ast}\in[0,\overline z]$, we still have \eqref{eq:zstar-Delta}. }
    \begin{equation}\label{eq:zstar-Delta}
        |z^{\ast}-z^{(k)\ast}|>\bar{\Delta}_z^{(k+1)}/2.
    \end{equation}
    In phase $k$, let $\bar{j}$ be the index of the grid point for $z^{(k)\ast}$.
    That is, $z^{(k)\ast} = \underline{z}^{(k)}+\bar j \delta_z^{(k)}$.
    Let $j^{\ast}$ be the index of the grid point in $[\underline z^{(k)}, \overline z^{(k)}]$ that is closest to $z^{\ast}$ among all grid points.
    Conditional on $B_k$, we have $z^{\ast}\in [\underline z^{(k)}, \overline z^{(k)}]$.
    Hence
    \begin{equation}\label{eq:zstar-delta}
        |\underline{z}^{(k)}+j^{\ast} \delta_z^{(k)}-z^{\ast}|=|z_{j^{\ast}}^{(k)}-z^{\ast}|\le \delta_z^{(k)}/2.
    \end{equation}
    where $z_{j}^{(k)}\triangleq \underline{z}^{(k)}+j \delta_z^{(k)}$.
%    Combining equations~\eqref{eq:zstar-Delta} and \eqref{eq:zstar-delta}, we have
%    \begin{equation}
%        |z_{\bar j}^{(i)} - z_{j^{\ast}}^{(i)}|\ge \max\left\{ \frac{\bar\Delta_z^{(i+1)}}{2}- \frac{\delta_z^{(i)}}{2}, 0\right\}
%    \end{equation}
%    Combining~\eqref{eq:zstar-Delta} and \eqref{eq:zstar-delta}, we have that $|j^{\ast}-\bar j| \delta_z^{(k)}>(\bar{\Delta}_z^{(k+1)}-\delta_z^{(k)})/2$.
    According to the way $z^{(k)\ast}$ is defined (see Step~\ref{step:dual} in the algorithm and \eqref{eq:dual_algorithm}),
    \begin{equation}\label{eq:jbar-smaller}
        cz_{\bar j}^{(k)} + T\sum_{m=1}^M\max_{j_m = 0,\ldots, N^{(k)}} \hat d_{m,j_m}^{(k)}(p_{m,j_m}^{(k)}-z_{\bar j}^{(k)})\le cz_{j^{\ast}}^{(k)} + T\sum_{m=1}^M\max_{j_m = 0,\ldots, N^{(k)}} \hat d_{m,j_m}^{(k)}(p_{m,j_m}^{(k)}-z_{j^{\ast}}^{(k)}).
    \end{equation}
    Furthermore, \eqref{eq:zstar-Delta}, \eqref{eq:zstar-delta} and \eqref{eq:jbar-smaller} imply that
    \begin{align*}
        B_{k+1}^c\implies &\left(cz_{j^{\ast}}^{(k)} + T\sum_{m=1}^M\max_{j_m = 0,\ldots, N^{(k)}} \hat d_{m,j_m}^{(k)}(p_{m,j_m}^{(k)}-z_{j^{\ast}}^{(k)})-g(z_{j^{\ast}}^{(k)})\right)\\
                                        &-\left(cz_{\bar j}^{(k)} + T\sum_{m=1}^M\max_{j_m = 0,\ldots, N^{(k)}} \hat d_{m,j_m}^{(k)}(p_{m,j_m}^{(k)}-z_{\bar j}^{(k)})-g(z_{\bar j}^{(k)})\right)\\
                                        &\ge g(z_{\bar j}^{(k)})-g(z_{j^{\ast}}^{(k)})= (g(z_{\bar j}^{(k)})-g(z^{\ast}))-(g(z_{j^{\ast}}^{(k)})-g(z^{\ast}))\\
                                        &\ge M_3 (z_{\bar j}^{(k)}-z^{\ast})^2- M_4(z_{j^{\ast}}^{(k)}-z^{\ast})^2\\
                                        &\ge \frac{M_3 }{4}(\bar\Delta_z^{(k+1)})^2-  \frac{M_4}{4}(\delta_z^{(k)})^2
    \end{align*}
    The last but one inequality is due to Remark~\ref{rmk:dual_convexity}; the last inequality is due to \eqref{eq:zstar-Delta} and \eqref{eq:zstar-delta}.
    If the difference of two terms is bounded below by $\frac{M_3 }{4}(\bar\Delta_z^{(k+1)})^2-  \frac{M_4}{4}(\delta_z^{(k)})^2$, then it implies that either the first term is greater than or equal to $\frac{M_3 }{8}(\bar\Delta_z^{(k+1)})^2-  \frac{M_4}{8}(\delta_z^{(k)})^2$, or the second term is less than or equal to $-\frac{M_3 }{8}(\bar\Delta_z^{(k+1)})^2+  \frac{M_4}{8}(\delta_z^{(k)})^2$.
    Therefore, noting that $B_k\cap C_k\in \mathcal F_{t_k}$ and taking the union on $\bar j$, we have
    \begin{align}\label{eq:union-zj}
        &\PR(B_{k+1}^c|B_k\cap C_k)= \E\left[\PR(B_{k+1}^c|\mathcal F_{t_k})| B_k\cap C_k\right]\notag\\
        &\le \E\bigg[\PR\bigg(\cup_{\bar j=0}^{N_z^{(k)}} \bigg\{cz_{j^{\ast}}^{(k)} + T\sum_{m=1}^M\max_{j_m = 0,\ldots, N^{(k)}} \hat d_{m,j_m}^{(k)}(p_{m,j_m}^{(k)}-z_{j^{\ast}}^{(k)})-g(z_{j^{\ast}}^{(k)})\notag\\
        &\quad\quad\ge \frac{M_3 }{8}(\bar\Delta_z^{(k+1)})^2-  \frac{M_4}{8}(\delta_z^{(k)})^2\bigg\}\bigg| \mathcal F_{t_k}\bigg)\bigg| B_k\cap C_k\bigg]\notag\\
        &\quad +\E\bigg[\PR\bigg(\cup_{\bar j=0}^{N_z^{(k)}} \bigg\{cz_{\bar j}^{(k)} + T\sum_{m=1}^M\max_{j_m = 0,\ldots, N^{(k)}} \hat d_{m,j_m}^{(k)}(p_{m,j_m}^{(k)}-z_{\bar j}^{(k)})-g(z_{\bar j}^{(k)})\notag\\
        &\quad\quad\le -\frac{M_3 }{8}(\bar\Delta_z^{(k+1)})^2+  \frac{M_4}{8}(\delta_z^{(k)})^2\bigg\}\bigg| \mathcal F_{t_k}\bigg)\bigg| B_k\cap C_k\bigg].
    \end{align}
    %because the events are independent of $\mathcal F_{t_{k}}$ and thus $B_k\cap C_k$.
    By the union bound, the term in the first expectation of \eqref{eq:union-zj} is bounded above by
    \begin{align}\label{eq:stoch-discrete}
        &\PR\bigg(\cup_{\bar j=0}^{N_z^{(k)}} \bigg\{cz_{j^{\ast}}^{(k)} + T\sum_{m=1}^M\max_{j_m = 0,\ldots, N^{(k)}} \hat d_{m,j_m}^{(k)}(p_{m,j_m}^{(k)}-z_{j^{\ast}}^{(k)})-g(z_{j^{\ast}}^{(k)})\notag\\
        &\quad\ge \frac{M_3 }{8}(\bar\Delta_z^{(k+1)})^2-  \frac{M_4}{8}(\delta_z^{(k)})^2\bigg\}\bigg| \mathcal F_{t_k}\bigg)\notag\\
        \le & (N^{(k)}_z+1) \PR\bigg( T\sum_{m=1}^M\max_{j_m = 0,\ldots, N^{(k)}} \hat d_{m,j_m}^{(k)}(p_{m,j_m}^{(k)}-z_{j^{\ast}}^{(k)})-T\sum_{m=1}^M\max_{p_m \in [\underline p,\overline p]} d_{m}(p_m)(p_{m}-z_{j^{\ast}}^{(k)})\notag\\
        &\quad  \ge \frac{M_3 }{8}(\bar\Delta_z^{(k+1)})^2-  \frac{M_4}{8}(\delta_z^{(k)})^2\bigg| \mathcal F_{t_k}\bigg)\notag\\
        \le & (N^{(k)}_z+1) \PR\bigg( T\sum_{m=1}^M\max_{j_m = 0,\ldots, N^{(k)}} \hat d_{m,j_m}^{(k)}(p_{m,j_m}^{(k)}-z_{j^{\ast}}^{(k)})-T\sum_{m=1}^M\max_{j_m = 0,\ldots, N^{(k)}} d_{m}(p_{m,j_m}^{(k)})(p_{m,j_m}^{(k)}-z_{j^{\ast}}^{(k)})\notag\\
        &\quad  \ge \frac{M_3 }{16}(\bar\Delta_z^{(k+1)})^2-  \frac{M_4}{16}(\delta_z^{(k)})^2\bigg| \mathcal F_{t_k}\bigg)\notag\\
        &+(N^{(k)}_z+1) \PR\bigg( T\sum_{m=1}^M\max_{j_m = 0,\ldots, N^{(k)}} d_{m}(p_{m,j_m}^{(k)})(p_{m,j_m}^{(k)}-z_{j^{\ast}}^{(k)})-T\sum_{m=1}^M\max_{p_m\in [\underline p,\overline p]} d_{m}(p_{m})(p_{m}-z_{j^{\ast}}^{(k)})\notag\\
        &\quad\ge \frac{M_3 }{16}(\bar\Delta_z^{(k+1)})^2-  \frac{M_4}{16}(\delta_z^{(k)})^2\bigg| \mathcal F_{t_k}\bigg),
    \end{align}
    where in the last inequality, we have used the fact that $\{a+b>c\}\subset \{a>c/2\}\cup \{b>c/2\}$ and the union bound again.
    For the two probabilities in \eqref{eq:stoch-discrete},
    the first one is attributed to the \emph{stochastic error}: conditional on $\mathcal F_{t_k}$, $\hat d_{m,j_m}^{(k)}$ is a Poisson random variable which may deviate from its mean $d_{m}(p^{(k)}_{m,j_m})$.
    It is bounded above by
    \begin{align}\label{eq:stoch-error}
        &(N^{(k)}_z+1)\PR\bigg(\cup_{\bm j\in \left\{0,\ldots,N^{(k)}\right\}^M} \bigg\{T\sum_{m=1}^M \hat d_{m,j_m}^{(k)}(p_{m,j_m}^{(k)}-z_{j^{\ast}}^{(k)})-T\sum_{m=1}^Md_{m}(p_{m,j_m}^{(k)})(p_{m,j_m}^{(k)}-z_{j^{\ast}}^{(k)})\notag\\
        &\quad \ge \frac{M_3 }{16}(\bar\Delta_z^{(k+1)})^2-  \frac{M_4}{16}(\delta_z^{(k)})^2\bigg\}\bigg| \mathcal F_{t_k}\bigg)\notag\\
    \le & (N^{(k)}_z+1)\sum_{\bm j\in \left\{0,\ldots,N^{(k)}\right\}^M} \PR\bigg( T\sum_{m=1}^M \hat d_{m,j_m}^{(k)}(p_{m,j_m}^{(k)}-z_{j^{\ast}}^{(k)})-T\sum_{m=1}^Md_{m}(p_{m,j_m}^{(k)})(p_{m,j_m}^{(k)}-z_{j^{\ast}}^{(k)})\notag\\
        &\quad \ge \frac{M_3 }{16}(\bar\Delta_z^{(k+1)})^2-  \frac{M_4}{16}(\delta_z^{(k)})^2\bigg| \mathcal F_{t_k}\bigg)\notag\\
    \le & (N^{(k)}_z+1)\sum_{\bm j\in \left\{0,\ldots,N^{(k)}\right\}^M}\sum_{m=1}^M \PR\bigg( T\hat d_{m,j_m}^{(k)}(p_{m,j_m}^{(k)}-z_{j^{\ast}}^{(k)})-Td_{m}(p_{m,j_m}^{(k)})(p_{m,j_m}^{(k)}-z_{j^{\ast}}^{(k)})\notag\\
        &\quad \ge \frac{M_3 }{16M}(\bar\Delta_z^{(k+1)})^2-  \frac{M_4}{16M}(\delta_z^{(k)})^2\bigg| \mathcal F_{t_k}\bigg)\notag\\
    \le & (N^{(k)}_z+1)\sum_{\bm j\in \left\{0,\ldots,N^{(k)}\right\}^M}\sum_{m=1}^M \PR\bigg( \bigg|\hat d_{m,j_m}^{(k)}-d_{m}(p_{m,j_m}^{(k)})\bigg|\ge \bigg|\frac{M_3 }{16MT(\overline p+\overline z)}(\bar\Delta_z^{(k+1)})^2\notag\\
        &\quad -  \frac{M_4}{16MT(\overline p+\overline z)}(\delta_z^{(k)})^2\bigg|\bigg| \mathcal F_{t_k}\bigg)
        \end{align}
        By the choice of parameters, for a sufficiently large $n$ (we require $(\log n)^\epsilon>\max\left\{M_4/M_3,2\right\}$), we have $M_4(\delta_z^{(k)})^2\le M_3 (\bar\Delta_z^{(k+1)})^2 (\log n)^{-\epsilon}\le M_3 (\bar\Delta_z^{(k+1)})^2/2$.
        Therefore, for a sufficiently large $n$ (below we require $(\log n)^{\epsilon/2}\ge \max\{32MT(\overline p+\overline z)/M_3, \sqrt{2}\}$ and $n\ge 3$  ) and the parameter values chosen in Section~\ref{sec:parameter}, we have
        \begin{align*}
            &\frac{M_3 }{16MT(\overline p+\overline z)}(\bar\Delta_z^{(k+1)})^2-  \frac{M_4}{16MT(\overline p+\overline z)}(\delta_z^{(k)})^2\ge \frac{M_3 }{32MT(\overline p+\overline z)}(\bar\Delta_z^{(k+1)})^2\\
            &\ge (\log n)^{-4.5\epsilon } n^{-(1/2)(1-(3/5)^{k})}\ge (\log n)^{1/2+\epsilon} \left(\frac{N^{(k)}+1}{n\tau^{(k)}} \right)^{1/2}.
        \end{align*}
        Also note that the Poisson random variable $D^{(k)}_{m,j} $ conditional on $\mathcal F_{t_k}$ satisfies
        \begin{equation*}
            \hat d^{(k)}_{m,j} \triangleq  \frac{N^{(k)}+1}{n\tau^{(k)}}D_{m,j}^{(k)}, \quad \E[D_{m,j}^{(k)}]=\frac{n\tau^{(k)}}{N^{(k)}+1} d_{m}(p_{m,j_m}^{(k)}).
        \end{equation*}
        Therefore, we apply Lemma~\ref{lem:poisson} by letting $r_n= \E[D^{(k)}_{m,j}]$ and $\epsilon_n = (\log n)^{1/2+\epsilon} r_n^{-1/2}(d_{m}(p_{m,j_m}^{(k)}))^{-1/2}$.
        Clearly, $r_n$ has polinomial growth and thus
        \begin{align}
            \PR\bigg( \bigg|\hat d_{m,j_m}^{(k)}-d_{m}(p_{m,j_m}^{(k)})\bigg|\ge \bigg|\frac{M_3 }{16MT(\overline p+\overline z)}(\bar\Delta_z^{(k+1)})^2-  \frac{M_4}{16MT(\overline p+\overline z)}(\delta_z^{(k)})^2\bigg|\bigg| \mathcal F_{t_k}\bigg)\notag\\
            \le \PR\bigg( D_{m,j}^{(k)}-\E[D_{m,j}^{(k)}|\mathcal F_{t_k}]\ge (\log n)^{1/2+\epsilon} \left(\frac{N^{(k)}+1}{n\tau^{(k)}} \right)^{-1/2}\bigg| \mathcal F_{t_k}\bigg)=O(n^{-(1+2M)}).
            \label{eq:apply-poisson-lemma}
        \end{align}
        (In fact, the constant in $O(\cdot)$ can be replaced by 1 if $(\log n)^{\epsilon}>4(1+2M)$ by the proof of Lemma~\ref{lem:poisson}.)
        Thus, \eqref{eq:stoch-error} and thus the first term of \eqref{eq:stoch-discrete} can be bounded
        \begin{equation*}
            \eqref{eq:stoch-error}\le M(N^{(k)}_z+1)(N^{(k)}+1)^M\times O\left( n^{-(1+ 2M)}\right) = O( n^{-1}).
        \end{equation*}
        The second equality is due to the fact that $N^{(k)}=O(n)$ and $N_z^{(k)}=O(n)$ for all $k$.

        The second probability in \eqref{eq:stoch-discrete} is attributed to the \emph{discretization error}, as we only sample discrete prices to find the optimal solution.
        The second probability is zero, because the optimal value is always greater than the discretized one given that $\frac{M_3 }{16}(\bar\Delta_z^{(k+1)})^2-  \frac{M_4}{16}(\delta_z^{(k)})^2>0$ for a sufficiently large $n$ (we require $(\log n)^{2\epsilon}>M_4/M_3$).

        For the second term in \eqref{eq:union-zj}, we can decompose it into the stochastic error and the discretization error similar to \eqref{eq:stoch-discrete}. The stochastic error can be bounded by $O(n^{-1})$ similarly.
        For the discretization error, we are going to bound
        \begin{align*}
            &\sum_{\bar j=0}^{N_z^{(k)}} \PR\bigg( T\sum_{m=1}^M\max_{j_m = 0,\ldots, N^{(k)}} d_{m}(p_{m,j_m}^{(k)})(p_{m,j_m}^{(k)}-z_{\bar j}^{(k)})-T\sum_{m=1}^M\max_{p_m\in [\underline p,\overline p]} d_{m}(p_{m})(p_{m}-z_{\bar j}^{(k)})\\
            &\quad \le -\frac{M_3 }{16}(\bar\Delta_z^{(k+1)})^2+  \frac{M_4}{16}(\delta_z^{(k)})^2\bigg| \mathcal F_{t_k}\bigg)
        \end{align*}
        conditional on $B_k\cap C_k$.
        Let $\bar j_m$ be the index for the grid point $p_{m,\bar j_m}^{(k)}$ that is closest to $\mathcal P_m(z_{\bar j}^{(k)})$.
        Conditional on $C_k$, the distance between $p_{m,\bar j_m}^{(k)}$ and $\mathcal P_m(z_{\bar j}^{(k)})$, which falls into $[\underline p_m^{(k)},\overline p_m^{(k)}]$, is at most $\delta_m^{(k)}/2$.
        By the strict concavity and Lipschitz continuity in Assumption~\ref{asp:convexity},
        \begin{align*}
            & \max_{j_m = 0,\ldots, N^{(k)}} d_{m}(p_{m,j_m}^{(k)})(p_{m,j_m}^{(k)}-z_{\bar j}^{(k)})-\max_{p_m\in [\underline p,\overline p]} d_{m}(p_{m})(p_{m}-z_{\bar j}^{(k)})\\
        \ge & - M_4 (d_m(p_{m,\bar j_m}^{(k)})- d_m(p^{\ast}_m(z_{\bar j}^{(k)})))^2\ge - M_4M_2^2 (p_{m,\bar j_m}^{(k)}- p^{\ast}_m(z_{\bar j}^{(k)}))^2\\
        \ge&-\frac{ M_4M_2^2}{4}(\delta_m^{(k)})^2
        \end{align*}
        By the choice of parameters, the grid size is sufficiently small, and thus
        \begin{align*}
            &T\sum_{m=1}^M\max_{j_m = 0,\ldots, N^{(k)}} d_{m}(p_{m,j_m}^{(k)})(p_{m,j_m}^{(k)}-z_{\bar j}^{(k)})-T\sum_{m=1}^M\max_{p_m\in [\underline p,\overline p]} d_{m}(p_{m})(p_{m}-z_{\bar j}^{(k)})\\
            \ge&-\frac{ TMM_4M_2^2}{4}(\delta_m^{(k)})^2\ge -n^{-(1/2)(1-(3/5)^{k})}(\log n)^{-5\epsilon}\ge-(\log n)^{-\epsilon/2} \frac{M_3 }{32}(\bar\Delta_z^{(k+1)})^2\\
             \ge &-(\log n)^{-\epsilon/2}\left(\frac{M_3 }{16}(\bar\Delta_z^{(k+1)})^2-  \frac{M_4}{16}(\delta_z^{(k)})^2\right)> -\frac{M_3 }{16}(\bar\Delta_z^{(k+1)})^2+  \frac{M_4}{16}(\delta_z^{(k)})^2.
        \end{align*}
        for a sufficiently large $n$ (we require $(\log n)^{\epsilon}> \max\{TMM_4M_2^2/4, (32/M_3)^2, (2M_4/M_3)^{1/2},1\}$ above).
        In other words, if $n$ is sufficiently large, the discretization error
        \begin{align*}
            &\sum_{\bar j=0}^{N_z^{(k)}} \PR\bigg( T\sum_{m=1}^M\max_{j_m = 0,\ldots, N^{(k)}} d_{m}(p_{m,j_m}^{(k)})(p_{m,j_m}^{(k)}-z_{\bar j}^{(k)})-T\sum_{m=1}^M\max_{p_m\in [\underline p,\overline p]} d_{m}(p_{m})(p_{m}-z_{\bar j}^{(k)})\\
            &\quad \le -\frac{M_3 }{16}(\bar\Delta_z^{(k+1)})^2+  \frac{M_4}{16}(\delta_z^{(k)})^2\bigg| \mathcal F_{t_k}\bigg)=0
        \end{align*}
        conditional on $B_k\cap C_k$.
        Therefore, we have proved the result.
\end{proof}

\begin{proof}[Proof of Lemma~\ref{lem:Ci}:]
    Denote the following event by $D_{k,m}$
    \begin{equation*}
        D_{k,m}\triangleq |p_m^{(k)\ast}- p_m^{\ast}|\le (\log n)^{-\epsilon}\bar\Delta^{(k+1)}
    \end{equation*}
    We first show that with probability $1-O(1/n)$, $\cap_{m=1}^M D_{k,m}$ occurs.
    Let $\bar j_m$ be the index of the grid point for $p_m^{(k)\ast}$.
    Let $j_m^{\ast}$ be the index of the grid point that is closest to $p_m^{\ast}$.
    When $D_{k,m}^c$ occurs, we have
    \begin{equation}\label{eq:p-Delta}
        |p^{\ast}_m-p_m^{(k)\ast}|=|p^{\ast}_m-p_{m,\bar j_m}^{(k)}|>(\log n)^{-\epsilon}\bar{\Delta}^{(k+1)}.
    \end{equation}
    Conditional on $A_k\subset B_k\cap C_k$, $p_m^{\ast}$ falls into $[\underline p_m^{(k)},\overline p_m^{(k)}]$, and thus it is at most $\delta_m^{(k)}/2$ away from the closest grid point:
    \begin{equation}\label{eq:p-delta}
        |p^{\ast}_m - p_{m,j^{\ast}_m}^{(k)}|\le \delta_m^{(k)}/2.
    \end{equation}
    Therefore, by the definition of $p^{(k)\ast}_m$, we have
    \begin{align}\label{eq:eventD}
        &\hat d_{m,\bar j_m}^{(k)}(p_{m,\bar j_m}^{(k)}-z^{(k)\ast})\ge  \hat d_{m, j^{\ast}_m}^{(k)}(p_{m,j^{\ast}_m}^{(k)}-z^{(k)\ast})\notag\\
        \implies & (\hat d_{m,\bar j_m}^{(k)}-d_{m}(p_{m,\bar j_m}^{(k)}))(p_{m,\bar j_m}^{(k)}-z^{(k)\ast})-(\hat d_{m, j^{\ast}_m}^{(k)}-d_m(p_{m,j^{\ast}_m}^{(k)}))(p_{m,j^{\ast}_m}^{(k)}-z^{(k)\ast})\notag\\
                 &\ge d_m(p_{m,j^{\ast}_m}^{(k)})(p_{m,j^{\ast}_m}^{(k)}-z^{(k)\ast})-d_{m}(p_{m,\bar j_m}^{(k)})(p_{m,\bar j_m}^{(k)}-z^{(k)\ast})\notag\\
                 & = (d_m(p_{m,j^{\ast}_m}^{(k)})-d_{m}(p_{m,\bar j_m}^{(k)}))(z^{\ast}-z^{(k)\ast})+d_m(p_{m,j^{\ast}_m}^{(k)})(p_{m,j^{\ast}_m}^{(k)}-z^{\ast})-d_{m}(p_{m,\bar j_m}^{(k)})(p_{m,\bar j_m}^{(k)}-z^{\ast})\notag\\
                 & = (d_m(p_{m,j^{\ast}_m}^{(k)})-d_{m}(p_{m,\bar j_m}^{(k)}))(z^{\ast}-z^{(k)\ast})+d_m(p_{m,j^{\ast}_m}^{(k)})(p_{m,j^{\ast}_m}^{(k)}-z^{\ast}) -d_m(p_{m}^{\ast})(p_{m}^{\ast}-z^{\ast})\notag\\
                 &\quad + d_m(p_{m}^{\ast})(p_{m}^{\ast}-z^{\ast})-d_{m}(p_{m,\bar j_m}^{(k)})(p_{m,\bar j_m}^{(k)}-z^{\ast}).
    \end{align}
    The left-hand side of \eqref{eq:eventD} is the stochastic error caused by the Poisson arrival.
    We are going to show that if $D_{k,m}^c$ occurs, then the right-hand side is large so \eqref{eq:eventD} occurs with small probability.

    Because of Assumption~\ref{asp:convexity} and \eqref{eq:p-delta}, the second and third terms of the right-hand side of \eqref{eq:eventD} can be bounded by
    \begin{align*}
        d_m(p_{m,j^{\ast}_m}^{(k)})(p_{m,j^{\ast}_m}^{(k)}-z^{\ast}) -d_m(p_{m}^{\ast})(p_{m}^{\ast}-z^{\ast})\ge -M_4(p_{m,j^{\ast}_m}^{(k)}-p_m^{\ast})^2\ge -\frac{M_4}{4} (\delta_m^{(k)})^2
    \end{align*}
    Moreover,
    \begin{align}\label{eq:p-discrete}
        &(d_m(p_{m,j^{\ast}_m}^{(k)})-d_{m}(p_{m,\bar j_m}^{(k)}))(z^{\ast}-z^{(k)\ast})+d_m(p_{m}^{\ast})(p_{m}^{\ast}-z^{\ast})-d_{m}(p_{m,\bar j_m}^{(k)})(p_{m,\bar j_m}^{(k)}-z^{\ast})\notag\\
        \ge & - M^{-1}_2 |p_{m,j^{\ast}_m}^{(k)}-p_{m,\bar j_m}^{(k)}|| z^{\ast}-z^{(k)\ast}|+M_3(d_m(p_{m,\bar j_m}^{(k)})-d_m(p_m^{\ast}))^2\notag\\
    \ge & - M^{-1}_2 |p_{m,j^{\ast}_m}^{(k)}-p_{m,\bar j_m}^{(k)}|| z^{\ast}-z^{(k)\ast}|+M_3M_2^{-2}(p_{m,\bar j_m}^{(k)}-p_m^{\ast})^2
    \end{align}
    In the first inequality, we use Lipschitz continuity (Assumption~\ref{asp:convexity}) for the first term; for the second term, note that $d_m(p^{\ast}_m)$ maximizes $\lambda(d^{-1}_m(\lambda)-z^{\ast})$ and its second-order derivative is bounded between $[-M_4,-M_3]$ by Assumption~\ref{asp:convexity}.
    The second inequality follows from the Lipschitz continuity (Assumption~\ref{asp:convexity}).
    Because $j_m^{\ast}$ is the closest grid point to $p_m^{\ast}$, we have
    \begin{align*}
        |p_{m,j^{\ast}_m}^{(k)}-p_{m,\bar j_m}^{(k)}|&\le |p_{m}^{\ast}-p_{m,\bar j_m}^{(k)}|+|p_{m,j^{\ast}_m}^{(k)}-p_{m}^{\ast}|\le 2|p_{m}^{\ast}-p_{m,\bar j_m}^{(k)}|.
    \end{align*}
%    because $|p_{m}^{\ast}-p_{m,\bar j_m}^{(k)}|\ge \log(n)^{-\epsilon}\bar\Delta^{(k+1)}>\delta_m^{(k)}/2$ for a sufficiently large $n$.
    Therefore, if $D_{k,m}^c$ occurs, then
    \begin{align*}
        \eqref{eq:p-discrete}\ge &|p_{m}^{\ast}-p_{m,\bar j_m}^{(k)}|(M_3M_2^{-2}|p_m^{\ast}-p_{m,\bar j_m}^{(k)}|- 2M_2^{-1}| z^{\ast}-z^{(k)\ast}|)\\
        \ge & \frac{\bar\Delta^{(k+1)}}{(\log n)^{2\epsilon}} \left(M_3M_2^{-2}\bar\Delta^{(k+1)}-2M_2^{-1} (\log n)^{\epsilon}\bar\Delta_z^{(k+1)}\right).
    \end{align*}
    In the inequality,
%    Moreover, for the first term of \eqref{eq:eventD}, by Lipschitz continuity in Assumption~\ref{asp:convexity} and \ref{asp:dual_convexity}, we have
    we have used the fact that $| z^{\ast}-z^{(k)\ast}|\le \bar\Delta_z^{(k+1)}$ conditional on $B_{k+1}$.
    Therefore, we have
    \begin{align}\label{eq:Dkm-two-prob}
        &\PR( D_{k,m}^c|B_k\cap C_k\cap B_{k+1})\notag\\
    \le & \PR\bigg(\cup_{\bar j_m=0}^{N^{(k)}}\bigg\{ (\hat d_{m,\bar j_m}^{(k)}-d_{m}(p_{m,\bar j_m}^{(k)}))(p_{m,\bar j_m}^{(k)}-z^{(k)\ast})-(\hat d_{m, j^{\ast}_m}^{(k)}-d_m(p_{m,j^{\ast}_m}^{(k)}))(p_{m,j^{\ast}_m}^{(k)}-z^{(k)\ast})\notag\\
        &\quad \ge \frac{\bar\Delta^{(k+1)}}{(\log n)^{2\epsilon}} \left(M_3M_2^{-2}\bar\Delta^{(k+1)}-2M_2^{-1} (\log n)^{\epsilon}\bar\Delta_z^{(k+1)}\right)-\frac{M_4}{4} (\delta_m^{(k)})^2\bigg\}\bigg| B_k\cap C_k\cap B_{k+1}\bigg)\notag\\
    \le & \sum_{j=0}^{N^{(k)}}\PR\bigg(\bigg|\hat d_{m,j}^{(k)}-d_{m}(p_{m,j}^{(k)})\bigg|\ge \bigg|\frac{\bar\Delta^{(k+1)}}{2(\overline p+\overline z)(\log n)^{2\epsilon}} \left(M_3M_2^{-2}\bar\Delta^{(k+1)}-2M_2^{-1} (\log n)^{\epsilon}\bar\Delta_z^{(k+1)}\right)\notag\\
        &\quad -\frac{M_4}{8(\overline p+\overline z)} (\delta_m^{(k)})^2\bigg|\bigg|B_k\cap C_k\cap B_{k+1}\bigg)+(N^{(k)}+1)\PR\bigg(\bigg|\hat d_{m, j^{\ast}_m}^{(k)}-d_m(p^{(k)}_{m,j^{\ast}_m})\bigg|\notag\\
        & \ge \bigg|\frac{\bar\Delta^{(k+1)}}{2(\log n)^{2\epsilon}(\overline p+\overline z)} \left(M_3M_2^{-2}\bar\Delta^{(k+1)}-2M_2^{-1} (\log n)^{\epsilon}\bar\Delta_z^{(k+1)}\right)\notag\\
        &\quad -\frac{M_4}{8(\overline p+\overline z)} (\delta_m^{(k)})^2\bigg|\bigg|B_k\cap C_k\cap B_{k+1}\bigg),
    \end{align}
    where the second inequality follows from $|p_{m,j}^{(k)}-z^{(k)\ast}|\le \overline p+\overline z$.
    By Lemma~\ref{lem:Bi}, $\PR(B_{k+1}|B_k\cap C_k)=1-O(n^{-1})$. Therefore, $\PR(\cdot | B_k\cap C_k\cap B_{k+1})-\PR(\cdot |B_k\cap C_k)=O(n^{-1})$.
    Therefore, by the fact that $B_{k}\cap C_k\in \mathcal F_{t_k}$, the first probability of \eqref{eq:Dkm-two-prob} is bounded above by
    \begin{align}\label{eq:Dkm-middle-step}
        &\sum_{j=0}^{N^{(k)}}\PR\bigg(\bigg|\hat d_{m,j}^{(k)}-d_{m}(p_{m,j}^{(k)})\bigg|\ge \bigg|\frac{\bar\Delta^{(k+1)}}{2(\overline p+\overline z)(\log n)^{2\epsilon}} \left(M_3M_2^{-2}\bar\Delta^{(k+1)}-2M_2^{-1} (\log n)^{\epsilon}\bar\Delta_z^{(k+1)}\right)\notag\\
        &\quad -\frac{M_4}{8(\overline p+\overline z)} (\delta_m^{(k)})^2\bigg|\bigg|B_k\cap C_k\cap B_{k+1}\bigg)\notag\\
        \le & \sum_{j=0}^{N^{(k)}}\E\bigg[\PR\bigg(\bigg|\hat d_{m,j}^{(k)}-d_{m}(p_{m,j}^{(k)})\bigg|\ge \bigg|\frac{\bar\Delta^{(k+1)}}{2(\overline p+\overline z)(\log n)^{2\epsilon}} \left(M_3M_2^{-2}\bar\Delta^{(k+1)}-2M_2^{-1} (\log n)^{\epsilon}\bar\Delta_z^{(k+1)}\right)\notag\\
            &\quad -\frac{M_4}{8(\overline p+\overline z)} (\delta_m^{(k)})^2\bigg|\bigg|\mathcal F_{t_k}\bigg)\bigg|B_k\cap C_k\bigg]+O(n^{-1}).
    \end{align}
    For $n\ge 3$, we have $\bar\Delta^{(k+1)}\ge (\log n)^{2\epsilon}\delta_m^{(k)}$ and thus
    \begin{align*}
        &\frac{\bar\Delta^{(k+1)}}{2(\overline p+\overline z)(\log n)^{2\epsilon}} \left(M_3M_2^{-2}\bar\Delta^{(k+1)}-2M_2^{-1} (\log n)^{\epsilon}\bar\Delta_z^{(k+1)}\right)-\frac{M_4}{8(\overline p+\overline z)} (\delta_m^{(k)})^2\\
        \ge& \frac{M_3M_2^{-2}(\bar\Delta^{(k+1)})^2}{4(\overline p+\overline z)(\log n)^{2\epsilon}}-\frac{M_4}{8(\overline p+\overline z)}(\delta_m^{(k)})^2\ge   \frac{M_3M_2^{-2}(\bar\Delta^{(k+1)})^2}{8(\overline p+\overline z)(\log n)^{2\epsilon}}\\
        \ge& (\log n)^{-3\epsilon} n^{-(1/2)(1-(3/5)^k)} \ge (\log n)^{1/2+\epsilon} \left(\frac{N^{(k)}+1}{n\tau^{(k)}} \right)^{1/2}.
    \end{align*}
    We require $(\log n)^{\epsilon}\ge 8(\bar p+\bar z)M_2^2/M_3 $ above.
    Note that conditional on $\mathcal F_{t_k}$, $D_{m,j}^{(k)}$ is a Poisson random variable and
    \begin{align*}
        \hat d^{(k)}_{m,j} \triangleq  \frac{N^{(k)}+1}{n\tau^{(k)}}D_{m,j}^{(k)}, \quad \E[D_{m,j}^{(k)}]=\frac{n\tau^{(k)}}{N^{(k)}+1} d_{m}(p_{m,j}^{(k)}).
    \end{align*}
    Therefore, by Lemma~\ref{lem:poisson}, similar to the setup in \eqref{eq:apply-poisson-lemma}, we have that \eqref{eq:Dkm-middle-step} is $O(n^{-1})$ and so is the second probability of \eqref{eq:Dkm-two-prob}. Hence $\PR( D_{k,m}^c|B_k\cap C_k\cap B_{k+1})=O(n^{-1})$ and
    $\PR(\cap_{m=1}^M D_{k,m}|B_k\cap C_k\cap B_{k+1})=1-O(n^{-1})$.

    Next we show that conditional on $\cap_{m=1}^M D_{k,m}$ and $ B_{k+1}$, $\PR(C_{k+1})=1-O(n^{-1})$.
    By Remark~\ref{rmk:dual_convexity}, for $z\in [\underline z^{(k+1)},\overline z^{(k+1)}]$,
    \begin{align*}
        |\mathcal P_m(z)-p_m^{\ast}| = |\mathcal P_m(z)-\mathcal P_m(z^{\ast})| \le M_2|z-z^{\ast}|\le M_2 \bar\Delta_z^{(k+1)},
    \end{align*}
    where we have used the fact that $z^{\ast}\in [\underline z^{(k+1)},\overline z^{(k+1)}]$ conditional on $B_{k+1}$.
    Conditional on $D_{k,m}$, we have
    \begin{align*}
        |\mathcal P_m(z)-p_m^{(k)\ast}|\le |\mathcal P_m(z)-p_m^{\ast}|+|p_m^{\ast}-p_m^{(k)\ast}|\le M_2 \bar\Delta_z^{(k+1)}+(\log n)^{-\epsilon}\bar\Delta^{(k+1)}\le \bar\Delta^{(k+1)}/2
    \end{align*}
    for a sufficiently large $n$ (we require $(\log n)^\epsilon >\max\{2\sqrt{M_4}, 4\}$).
    Therefore, we have shown that for all $z\in [\underline z^{(k+1)},\overline z^{(k+1)}]$, $\mathcal P_m(z)\in [\underline p_m^{(k+1)},\overline p_m^{(k+1)}]$ with probability $1-O(n^{-1})$ and this completes the proof.
\end{proof}

\begin{proof}[Proof of Lemma~\ref{lem:bi-ci}:]
    Note that because $A_k\supseteq B_k\cap C_k$, we have
    \begin{align*}
    \PR\left(\left\{\cap_{k=1}^K\left\{A_k\cap B_k\cap C_k\}\right\}\right\}^c\right)  = \PR\left(\cup_{k=1}^K\left\{ B^c_k\cup C^c_k\right\}\right)
                                                                                      \le \sum_{k=1}^K \PR\left(B^c_k\cup C_k^c\right)
    \end{align*}
    Note that by Assumption~\ref{asp:initial}, $\PR(B_1^c\cup C_1^c)=0$.
    For $k\ge 2$, we have
    \begin{align*}
        \PR\left(B^c_k\cup C_k^c\right)&\le \PR(B^c_k\cup C^c_k| B_{k-1}\cap C_{k-1})\PR(B_{k-1}\cap C_{k-1})+\PR(B^c_{k-1}\cup C^c_{k-1})\\
                                       &\le  \PR(B_k^c|B_{k-1}\cap C_{k-1})+\PR(B_k\cap C_{k-1}^c|B_{k-1}\cap C_{k-1})+\PR(B^c_{k-1}\cup C^c_{k-1})\\
                                       &\le \PR(B_k^c|B_{k-1}\cap C_{k-1})+\PR(C_{k-1}^c|B_k\cap B_{k-1}\cap C_{k-1})+\PR(B^c_{k-1}\cup C^c_{k-1})
%                                       &\le  \left(\PR(B^c_k| B_{k-1}\cap C_{k-1})+\PR(C^c_k| B_{k-1}\cap C_{k-1})\right)\left(1-\PR(B_{k-1}^c| B_{i-2}\cap C_{i-2})-\PR(C_{k-1}^c| B_{i-2}\cap C_{i-2})\right)
    \end{align*}
    which implies $\PR\left(B^c_k\cup C_k^c\right)-\PR(B^c_{k-1}\cup C^c_{k-1}) = O(1/n)$
    by Lemma~\ref{lem:Bi} and Lemma~\ref{lem:Ci}.
    Repeating the process until $k=2$, we have $\PR\left(B^c_k\cup C_k^c\right)= k\times O(1/n)$.
    Therefore, by the fact that $k\le K=O(\log n)$ from Lemma~\ref{lem:K}, we have
    \begin{align*}
        &\sum_{k=1}^K \frac{\PR\left(B^c_k\cup C_k^c\right)}{K^2}\le \frac{1}{K}\sum_{k=1}^K \frac{\PR\left(B^c_k\cup C_k^c\right)}{i}= O(1/n)\\
        \implies & \sum_{k=1}^K \PR\left(B^c_k\cup C_k^c\right)=O((\log n)^2/n).
    \end{align*}
    This completes the proof.
\end{proof}

\begin{proof}[Proof of Lemma~\ref{lem:inventory_tK1}:]
    If the event $\cap_{k=1}^{K}A_k$ occurs, then we have
    \begin{equation*}
        |P_m(t) -p_m^{\ast}| \le \bar\Delta^{(k)} = n^{-(1/4)(1-(3/5)^{k-1})},
    \end{equation*}
    for $t\in (t_k,t_{k+1}]$.
    Therefore, $\cap_{k=1}^{K}A_k$ implies that
    \begin{align*}
         \int_0^{t_K}\sum_{m=1}^Md_m(P_m(t))dt-\sum_{k=1}^{K-1}\tau^{(k)}\sum_{m=1}^Md_m(p_m^{\ast})\le &
         M_2\sum_{k=1}^{K-1}\tau^{(k)} \sum_{m=1}^M|P_m(t)-p_m^{\ast}|\\
         \le & M_2 M\sum_{k=1}^{K-1}\tau^{(k)}\bar\Delta^{(k)}.
    \end{align*}
    By the choice of parameters, $\bar\Delta^{(k)}\ge n^{-1/2}(\log n)^{2+16\epsilon}$ for $k\le K-1$, which implies that $n^{1/4\times (3/5)^{k-1}}\ge(\log n)^{1+8\epsilon}$ for $k\le K-1$.
    Moreover,
    $\tau^{(k)}\bar\Delta^{(k)}\le n^{-1/4(1+(3/5)^{k-1})}(\log n)^{1+15\epsilon}\le n^{-1/4}(\log n)^{7\epsilon}$.
    Therefore, by Lemma~\ref{lem:K},
    \begin{align*}
        M_2M\sum_{k=1}^{K-1}\tau^{(k)}\bar\Delta^{(k)}\le M_2MK n^{-1/4} (\log n)^{7\epsilon}< n^{-1/4}(\log n)^{1+8\epsilon}
    \end{align*}
    for a sufficiently large $n$ (we require $\log n>M_2M K$).
    Therefore,
    \begin{align*}
        &\PR\left(\left| \int_0^{t_K}\sum_{m=1}^Md_m(P_m(t))dt-\sum_{k=1}^{K-1}\tau^{(k)}\sum_{m=1}^Md_m(p_m^{\ast}) \right| >n^{-1/4}(\log n)^{1+8\epsilon} \right)\\
        \le& 1-\PR(\cap_{k=1}^{K}A_k)=O((\log n)^2/n)
    \end{align*}
    by Lemma~\ref{lem:bi-ci}.
    This completes the proof.
\end{proof}

\begin{proof}[Proof of Lemma~\ref{lem:inventory_tK2}:]
    Note that $S(t)$ can be expressed as
    \begin{align*}
        S(t) = \int_0^{t_K}\sum_{m=1}^MdN_{m,t}(nd_m(P_m(t))).
    \end{align*}
    Substituting it into Lemma~\ref{lem:inventory_tK2}, it suffices to show
    \begin{align}\label{eq:prob-inventory}
        \PR\left(\left| \int_0^{t_K}\sum_{m=1}^MdN_{m,t}(nd_m(P_m(t)))-n\sum_{k=1}^{K-1}\tau^{(k)}\sum_{m=1}^Md_m(p_m^{\ast}) \right| > 2n^{3/4}(\log n)^{1+8\epsilon} \right)
    \end{align}
    is $O((\log n)^2n^{-1/2})$.
    By the triangle inequality and the union bound, we have
    \begin{align}\label{eq:prob-inventory-two-terms}
        \eqref{eq:prob-inventory}\le &\PR\left(\left|\int_0^{t_K}\sum_{m=1}^MdN_{m,t}(nd_m(P_m(t)))-n \int_0^{t_K}\sum_{m=1}^Md_m(P_m(t))dt \right| > n^{3/4}(\log n)^{1+8\epsilon} \right)\notag\\
                                     &+\PR\left(\left| \int_0^{t_K}\sum_{m=1}^Md_m(P_m(t))dt-\sum_{k=1}^{K-1}\tau^{(k)}\sum_{m=1}^Md_m(p_m^{\ast}) \right| >n^{-1/4}(\log n)^{1+8\epsilon}  \right)
    \end{align}
    The second term is $O((\log n)^2/n)$ shown by Lemma~\ref{lem:inventory_tK1}.
    The first term can be bounded by
    \begin{align*}
        &\PR\left(\left|\int_0^{t_K}\sum_{m=1}^MdN_{m,t}(nd_m(P_m(t)))- n\int_0^{t_K}\sum_{m=1}^Md_m(P_m(t))dt \right| > n^{3/4}(\log n)^{1+8\epsilon}\right)\\
        \le & \sum_{k=1}^{K-1} \PR\left(\left|\int_{t_k}^{t^{k+1}}\sum_{m=1}^MdN_{m,t}(nd_m(P_m(t)))-n \int_{t_k}^{t_{k+1}}\sum_{m=1}^Md_m(P_m(t))dt \right| > \frac{n^{3/4}(\log n)^{1+8\epsilon}}{K}\right)\\
        \le &\sum_{k=1}^{K-1} \frac{1}{n^{3/2}(\log n)^{2+16\epsilon}} \E\left[\left(\int_{t_k}^{t_{k+1}}\sum_{m=1}^MdN_{m,t}(nd_m(P_m(t)))-n \int_{t_k}^{t_{k+1}}\sum_{m=1}^Md_m(P_m(t))dt \right)^2\right],
    \end{align*}
    where the last inequality follows from Chebyshev's inequality.
    Now note that conditional on $\mathcal F_{t_k}$, $P_m(t)$ for $t\in (t_k,t_{k+1}]$ is measurable and thus $\int_{t_k}^{t_{k+1}}dN_{m,t}(nd_m(P_m(t)))$ is a Poisson random variable with mean $n \int_{t_k}^{t_{k+1}}d_m(P_m(t))dt$.
    Therefore,
    \begin{align*}
        &\E\left[\left(\int_{t_k}^{t^{k+1}}\sum_{m=1}^MdN_{m,t}(nd_m(P_m(t)))-n \int_{t_k}^{t_{k+1}}\sum_{m=1}^Md_m(P_m(t))dt \right)^2\bigg| \mathcal F_{t_k}\right]\\
         =& n \int_{t_k}^{t_{k+1}}\sum_{m=1}^Md_m(P_m(t))dt\le nTM M_1 ,
    \end{align*}
    by Assumption~\ref{asp:convexity}.
    Thus, the first term of \eqref{eq:prob-inventory-two-terms} can be bounded by
    \begin{align*}
        \sum_{k=1}^{K-1} \frac{nTMM_1}{n^{3/2}(\log n)^{2+16\epsilon}}\le KTMM_1 n^{-1/2}(\log n)^{-2-16\epsilon}= O(n^{-1/2}(\log n)^{-2})
    \end{align*}
    when $(\log n)^{16\epsilon}>KTMM_1$.
    This completes the proof.
\end{proof}

\begin{proof}[Proof of Proposition~\ref{prop:regret-z<=0}:]
    The total revenue $\tilde{J}_n^{\pi}$ can be expressed as the sales revenue minus the outsourcing cost:
    \begin{align}\label{eq:z<0_all_regret}
        \E\left[\tilde{J}_{n}^{\pi}\right]&\ge  \E\left[ \int_0^T P_m(t)d N_{m,t}(nd_m(P_m(t)))\right] - \overline p\E\left[ \left(\sum_{m=1}^M\int_0^T d N_{m,t}(nd_m(P_m(t)))-nc\right)^+\right]\notag\\
                                          &=n \E\left[ \int_0^T P_m(t)d_m(P_m(t))dt\right] - \overline p\E\left[ \left(\sum_{m=1}^M\int_0^T d N_{m,t}(nd_m(P_m(t)))-nc\right)^+\right]\notag\\
                                          &= \sum_{k=1}^Kn\E\left[ \int_{t_k}^{t_{k+1}} P_m(t)d_m(P_m(t))dt\right] - \overline p\E\left[ \left(\sum_{m=1}^M\int_0^T d N_{m,t}(nd_m(P_m(t)))-nc\right)^+\right]\notag\\
                                          & \eqqcolon \sum_{k=1}^K\E[\tilde{J}_{n,k}^{\pi}] - \overline p\E\left[ \left(\sum_{m=1}^M\int_0^Td N_{m,t}(nd_m(P_m(t))) -nc\right)^+\right],
    \end{align}
    where we use $\tilde{J}_{n,k}^{\pi}$ to denote the revenue earned in phase $k$ (ignoring the inventory constraint).
    First note that
    \begin{align}\label{eq:cond_exp_similar}
        \E\left[\tilde{J}_{n,k}^{\pi}\big|A_k\right]-\E\left[\tilde{J}_{n,k}^{\pi}\right]\le& \PR(A_k^c)\E\left[\tilde{J}_{n,k}^{\pi}\big| A_k^c\right]\notag\\
        \le &O((\log n)^2 n^{-1}) \times \E\left[\E\left[\tilde{J}_{n,k}^{\pi}\big|\mathcal F_{t_k}\right]\big|A_k^c\right] = O((\log n)^2)
    \end{align}
    The second inequality follows from $\PR(A_k^c) = O((\log n)^2 n^{-1})$ by Lemma~\ref{lem:bi-ci},
    the tower property and the fact that $A_k^{c}\in \mathcal F_{t_k}$.
    The last bound is due to the fact that $\E\left[\tilde{J}_{n,k}^{\pi}\big|\mathcal F_{t_k}\right]$ is the mean of a Poisson random variable which is bounded above by $MM_1\overline p\tau^{(k)}n$.

%    and $\tilde{J}^{\pi}_{n,k}\le n\overline pc$ for all sample paths.\nc{not true, can use Poisson process to prove.}
    Since the difference is negligible compared to our target $n^{1/2}$, it suffices to study $\E\left[\tilde{J}_{n,k}^{\pi}\big|A_k\right]$.
    Conditional on $A_k$, we have $|P_m(t)-p_m^{\ast}|\le \bar\Delta^{(k)}=n^{-(1/4)(1-(3/5)^{k-1})}$ for $t\in(t_k,t_{k+1}]$ and $k\le K-1$.
    Note that the revenue generated in the fluid system using $p_m^{\ast}$ in phase $k$ is $n\tau^{(k)} \sum_{m=1}^Mp_m^{\ast}d_m(p_m^{\ast})$.
    Therefore, for $k\le K-1$,
    \begin{align}\label{eq:z<0_all_phases}
        n\tau^{(k)} \sum_{m=1}^Mp_m^{\ast}d_m(p_m^{\ast})-\E\left[\tilde{J}_{n,k}^{\pi}\big|A_k\right]&\le  n \E\left[ \int_{t_k}^{t_{k+1}} (p_m^{\ast}d_m(p_m^{\ast})-P_m(t)d_m(P_m(t)))dt|A_k\right]\notag\\
                                                                                                  &\le nM_4 \tau^{(k)} (\bar\Delta^{(k)})^2 =O(n^{1/2} (\log n)^{1+15\epsilon}).
    \end{align}
    The last inequality is because of the fact that $z^{\ast}\le 0$ and thus $p_m^{\ast}$ maximizes $p_m^{\ast}d_m(p_m^{\ast})$.
    The quadratic bound then follows from Assumption~\ref{asp:convexity}.
%    \nc{don't forget that case 1 is triggered with high probability, and $p_m^{\ast}$ can be more than $\bar\Delta^{(K)}$ away in the last phase.}

    For $k=K$, note that conditional on $B_K$, Step~\ref{step:sufficient_capacity} is triggered.
    Therefore, conditional on $A_K\cap B_K$, the charged price $ p_m^{(K)}$ in phase $K$ for type-$m$ consumers satisfies $|p_m^{(K)}-p_m^{\ast}|\le \bar\Delta^{(K)}+\alpha\le 2(\log n)^{1+9\epsilon}n^{-1/4}$.
    Therefore,
    \begin{align}\label{eq:z<0_last_phase}
        &n\tau^{(K)} \sum_{m=1}^Mp_m^{\ast}d_m(p_m^{\ast})-\E\left[\tilde{J}_{n,K}^{\pi}\right]\notag\\
        = &n\tau^{(K)} \sum_{m=1}^Mp_m^{\ast}d_m(p_m^{\ast})-\E\left[\tilde{J}_{n,K}^{\pi}\big| A_K\cap B_K\right]+(\E\left[\tilde{J}_{n,K}^{\pi}\big| A_K\cap B_K\right]-\E\left[\tilde{J}_{n,K}^{\pi}\right])\notag\\
    \le & n M_4\tau^{(K)}(\bar\Delta^{(K)}+\alpha)^2+O((\log n)^2)=O((\log n)^{2+18\epsilon}n^{1/2}).
    \end{align}
    The second inequality follows from a similar argument below \eqref{eq:cond_exp_similar}.
    Therefore, adding up \eqref{eq:z<0_all_phases} for all $K-1$ phases and \eqref{eq:z<0_last_phase}, we have obtained that
    \begin{align}\label{eq:JD_and_tildeJpi}
        J_n^D-\sum_{k=1}^{K}\E\left[\tilde{J}_{n,k}^{\pi}\right]= O(n^{1/2}(\log n)^{2+18\epsilon}).
    \end{align}

    From \eqref{eq:z<0_all_regret}, it remains to bound $\E\left[ \left(\sum_{m=1}^M\int_0^T d N_{m,t}(nd_m(P_m(t))) -nc\right)^+\right]$ by the same order.
    Consider the event
    \begin{equation*}
      E_K \triangleq A_K\cap \left\{\left| S(t_K)-n\sum_{k=1}^{K-1}\tau^{(k)}\sum_{m=1}^Md_m(p_m^{\ast}) \right| <2n^{3/4}(\log n)^{1+8\epsilon} \right\}.
    \end{equation*}
    By Lemma~\ref{lem:bi-ci} and Lemma~\ref{lem:inventory_tK2}, the probability of $E^c_K$ is $O(n^{-1/2})$.
    Therefore,
    \begin{align*}
        &\E\left[ \left(\sum_{m=1}^M\int_0^T d N_{m,t}(nd_m(P_m(t))) -nc\right)^+\right]\\
        \le& \E\left[ \left(\sum_{m=1}^M\int_0^T d N_{m,t}(nd_m(P_m(t))) -nc\right)^+\Inb{E_K}\right]+\E\left[ \left(\sum_{m=1}^M\int_0^T d N_{m,t}(nd_m(P_m(t))) +nc\right) \Inb{E_K^c}\right]\\
        \le& \E\left[ \left(\sum_{m=1}^M\int_0^T d N_{m,t}(nd_m(P_m(t))) -nc\right)^+\Inb{E_K}\right]\\
           &\quad+\left(\E\left[ \sum_{m=1}^M\int_0^T d N_{m,t}(nd_m(P_m(t)))\bigg|\PR(E_k^c) \right]+nc\right)\PR(E_k^c)\\
           \le& \E\left[ \left(\sum_{m=1}^M\int_0^T d N_{m,t}(nd_m(P_m(t))) -nc\right)^+\Inb{E_K}\right]+\E\left[X \Inb{E_k^c}\right]+nc\PR(E_k^c)
    \end{align*}
    Here $X$ is a Poisson random variable with mean $MM_1Tn$, which stochastically dominates $\sum_{m=1}^M\int_0^T d N_{m,t}(nd_m(P_m(t)))$.
    Centralizing $X$ and applying the Cauchy-Schwartz inequality, we have
    \begin{align*}
        \E\left[X \Inb{E_k^c}\right]\le \left(\Var(X)\PR(E_k^c)\right)^{1/2} + \E[X]\PR(E_k^c)= O(n^{1/2})
    \end{align*}
    Therefore, the terms above can be bounded by
    \begin{equation*}
        \E\left[ \left(\sum_{m=1}^M\int_0^T d N_{m,t}(nd_m(P_m(t))) -nc\right)^+\Inb{E_K}\right]+ C_1 n^{1/2}
    \end{equation*}
%           &\quad+\left(\E\left[ \sum_{m=1}^M\int_0^T d N_{m,t}(nd_m(P_m(t)))\bigg|\PR(E_k^c) \right]+nc\right)\PR(E_k^c)\\
%        \le& \E\left[ \left(\sum_{m=1}^M\int_0^T d N_{m,t}(nd_m(P_m(t))) -nc\right)^+\Inb{E_K}\right]+\left(\sqrt{(MM_1Tn)^2+MM_1Tn}+nc\right)\PR(E_K^c)\\
%        \le& ,
    for some constant $C_1$ that is independent of $n$ and $\epsilon$.
    When $E_K\subset A_K$ occurs, the prices set in phase $K$ satisfy (recall that $\alpha=(\log n)^{1+9\epsilon}n^{-1/4}$ by Section~\ref{sec:parameter})
    \begin{align*}
        \alpha\le  p_m^{(K)}-p_m^{\ast}\le \bar\Delta^{(K)}+\alpha\le 2(\log n)^{1+9\epsilon}n^{-1/4},
    \end{align*}
    according to the definition of $p_m^{(K)}$ and Assumption~\ref{asp:initial}.
    Moreover, by Assumption~\ref{asp:convexity}, $d_m(p_m^{\ast})-d_m(p_m^{(K)})\ge M_2^{-1}(p_m^{(K)}-p_m^{\ast})\ge M_2^{-1}\alpha=M_2^{-1}(\log n)^{1+9\epsilon}n^{-1/4}$.
    Therefore, combining with the fact that $c\ge T \sum_{m=1}^Md_m(p_m^{\ast})$, we have
    \begin{align}
        &\E\left[ \left(\sum_{m=1}^M\int_0^T d N_{m,t}(nd_m(P_m(t))) -nc\right)^+\Inb{E_K}\right]\notag\\
        \le &\E\left[ \left(\sum_{m=1}^M\int_0^T d N_{m,t}(nd_m(P_m(t))) -nT \sum_{m=1}^M d_m(p_m^{\ast})\right)^+\Inb{E_K}\right]\notag\\
        = &\E\left[ \left(S(t_K)-n\sum_{k=1}^{K-1}\tau^{(k)}\sum_{m=1}^M d_m(p_m^{\ast})+ \sum_{m=1}^M\int_{t_K}^T d N_{m,t}(nd_m(p_m^{(K)})) - n\tau^{(K)} \sum_{m=1}^M d_m(p_m^{\ast})\right)^+\Inb{E_K}\right]\notag\\
        \le& \E\left[ \left(2n^{3/4}(\log n)^{1+8\epsilon}+\sum_{m=1}^M\int_{t_K}^T d N_{m,t}(nd_m(p^{(K)}_m)) -n\tau^{(K)} \sum_{m=1}^Md_m(p_m^{\ast})\right)^+\Inb{E_K}\right]\notag\\
        \le& \E\bigg[ \bigg(2n^{3/4}(\log n)^{1+8\epsilon}- M_2^{-1}M(\log n)^{1+9\epsilon}n^{3/4}\tau^{(K)} +\sum_{m=1}^M\int_{t_K}^T d N_{m,t}(nd_m(p^{(K)}_m)) \notag\\
           &\quad -n\tau^{(K)} \sum_{m=1}^Md_m(p_m^{(K)})\bigg)^+\Inb{E_K}\bigg].
        \label{eq:bound-extra-inventory}
    \end{align}
    The first inequality is due to $c\ge T \sum_{m=1}^Md_m(p_m^{\ast})$.
    The second inequality is due to $E_K$.
    In the third inequality we replace $d_m(p_m^{\ast})$ by $d_m(p_m^{(K)})$ and use the bound derived before.
    Note that conditional on $\mathcal F_{t_K}$, $N_{m,t}(nd_m(p^{(K)}_m))$ is a Poisson process with rate $nd_m(p^{(K)}_m)$.
    Therefore, for a sufficiently large $n$ (we require $(\log n)^{\epsilon}>4M_2/TM$ below)
    \begin{align*}
        \eqref{eq:bound-extra-inventory}&\le  \E\bigg[ \bigg(\sum_{m=1}^M\int_{t_K}^T d N_{m,t}(nd_m(p^{(K)}_m))-n\tau^{(K)} \sum_{m=1}^Md_m(p_m^{(K)})\bigg)^+\bigg]\\
                                        &= \E\bigg[\E\bigg[ \bigg(\sum_{m=1}^M\int_{t_K}^T d N_{m,t}(nd_m(p^{(K)}_m))-n\tau^{(K)} \sum_{m=1}^Md_m(p_m^{(K)})\bigg)^+\bigg|\mathcal F_{t_K}\bigg]\bigg]\\
                                        &\le \E\bigg[ \left(n\tau^{(K)} \sum_{m=1}^Md_m(p_m^{(K)})\right)^{1/2}\bigg]=O(n^{1/2}).
    \end{align*}
    The first inequality is due to the fact that $2n^{3/4}(\log n)^{1+8\epsilon}- M_2^{-1}M(\log n)^{1+9\epsilon}n^{3/4}\tau^{(K)}\le 0$ for $(\log n)^{\epsilon}>4M_2/TM$ and by $\tau^{(K)}>T/2$ from Lemma~\ref{lem:total-length-exploration}.
    The second line is due to the tower property.
    The third line follows from Lemma~\ref{lem:poisson-expect}.
    Combining the above result with \eqref{eq:JD_and_tildeJpi} and \eqref{eq:z<0_all_regret}, we have proved Proposition~\ref{prop:regret-z<=0}.
\end{proof}

\begin{proof}[Proof of Lemma~\ref{lem:inventory-XT}:]
    Note that conditional on $B_K$, Step~\ref{step:insufficient_capacity} of the algorithm is triggered for a sufficiently large $n$ (we require $ z^{\ast}> n^{-1/4}(\log n)^{2+16\epsilon}> \bar\Delta^{(K)}_z$, which implies $0\notin [\underline z^{(K)},\overline z^{(k)}] $).
    Because $\PR(B_K)$ occurs with high probability (relative to the target regret), in the following we simply assume the algorithm enters Step~\ref{step:insufficient_capacity} in order to simplify the notation.
    We first bound the estimation error of $D^{(K)}_l$ and $D^{(K)}_u$.
    Conditional on $\mathcal F_{t_K}$,
    $p_m^l$ and $p_m^u$ are deterministic
    and $(\log n)^{-\epsilon}D^{(K)}_l$ is a Poisson random variable with mean $n(\log n)^{-\epsilon}\sum_{m=1}^M d_m(p_m^l)$.
    To apply Lemma~\ref{lem:poisson}, we set $r_n = n(\log n)^{-\epsilon}\sum_{m=1}^M d_m(p_m^l)$ and $\epsilon_n = (\log n)^{1+\epsilon/2} r_n^{-1/2} (\sum_{m=1}^M d_m(p_m^l))^{-1/2}$.
    Clearly, $r_n$ has polynomial growth in $n$. (In fact, the constant in $O(\cdot)$ below can be replaced by 1 if $\log n\ge 4$ and $n^{1/2}> (\log n)^{1+\epsilon}(\sum_{m=1}^M d_m(p_m^l))$.)
    As a result,
    \begin{equation}\label{eq:dD1}
        \PR\left(\bigg|(\log n)^{-\epsilon}D^{(K)}_l-n(\log n)^{-\epsilon}\sum_{m=1}^M d_m(p_m^l)\bigg|> \log(n)n^{1/2}\bigg| \mathcal F_{t_K}\right) = O(1/n).
    \end{equation}
    Similarly,
    \begin{equation} \label{eq:dD2}
        \PR\left(\bigg|(\log n)^{-\epsilon}D^{(K)}_u-n(\log n)^{-\epsilon}\sum_{m=1}^M d_m(p_m^u)\bigg|> \log(n)n^{1/2}\bigg| \mathcal F_{t_K}\right) = O(1/n).
    \end{equation}
    Introduce the event
    \begin{equation*}
        E_K = A_K\cap B_K\cap \left\{\left| S(t_K)-n\sum_{k=1}^{K-1}\tau^{(k)}\sum_{m=1}^Md_m(p_m^{\ast}) \right| <2n^{3/4}(\log n)^{1+8\epsilon} \right\}\in \mathcal F_{t_K}.
    \end{equation*}
    By Lemma~\ref{lem:bi-ci} and Lemma~\ref{lem:inventory_tK2}, $\PR(E_K^c)=O(n^{-1/2})$.
    Next we will show that conditional on $E_K$, $\theta\in (0,1)$ with high probability for a sufficiently large $n$.
    Introduce $\theta_0$, which solves
    \begin{equation*}
        n(T-t_K)\sum_{m=1}^M (\theta_0 d_m(p_m^l)+(1-\theta_0) d_m(p_m^u)) = nc - S(t_K).
    \end{equation*}
    $\theta_0$ is regarded as the ``expected'' version of $\theta$ compared to~\eqref{eq:theta}.
    Observe that on $E_k$ we have $\big|nc-S(t_K)-n(T-t_K)\sum_{m=1}^Md_m(p_m^{\ast})\big|<2n^{3/4}(\log n)^{1+8\epsilon}$.
    Moreover, on $A_K$, $p_m^{\ast}\in [\underline p^{(K)}_m,\overline p^{(K)}_m]$, and thus
    $\alpha\le p_m^{\ast}-p_m^{l}$ and $\alpha\le p_m^{u}-p_m^{\ast}$ where recall that $\alpha=(\log n)^{1+9\epsilon}n^{-1/4}$.
    Therefore, $\sum_{m=1}^M(d_m(p_m^{l})-d_m(p_m^{\ast}))\ge M_2^{-1}\sum_{m=1}^M(p_m^{\ast}-p_m^l)\ge M_2^{-1}M(\log n)^{1+9\epsilon}n^{-1/4}$ and $\sum_{m=1}^M(d_m(p_m^{\ast})-d_m(p_m^{u}))\ge M_2^{-1}M(\log n)^{1+9\epsilon}n^{-1/4}$.
    We have
    \begin{align*}
        n(T-t_K)\sum_{m=1}^M d_m(p_m^l) &> n(T-t_K)\sum_{m=1}^M d_m(p_m^{\ast})+3n^{3/4}(\log n)^{1+8\epsilon}\\
                                        &>  nc - S(t_K)+n^{3/4}\\
        n(T-t_K)\sum_{m=1}^M d_m(p_m^u) &< n(T-t_K)\sum_{m=1}^M d_m(p_m^{\ast})-3n^{3/4}(\log n)^{1+8\epsilon}\\
                                        &< nc - S(t_K)-n^{3/4}
    \end{align*}
    Therefore,
    \begin{align}\label{eq:theta0}
        \theta_0 &= \frac{nc - S(t_K)-n(T-t_K)\sum_{m=1}^M d_m(p_m^u)}{n(T-t_K)\sum_{m=1}^M (d_m(p_m^l)-d_m(p_m^u))}> \frac{n^{3/4}}{ n(T-t_k)MM_1} \ge  c_1 n^{-1/4}\\
        \theta_0 &= 1- \frac{n(T-t_K)\sum_{m=1}^M d_m(p_m^l)-(nc - S(t_K))}{n(T-t_K)\sum_{m=1}^M (d_m(p_m^l)-d_m(p_m^u))}<1- \frac{n^{3/4}}{ n(T-t_k)MM_1}\le 1-c_1n^{-1/4}\notag
    \end{align}
%    \nc{can show bounded away from 0}
    for some positive constant $c_1$ that is independent of $n$.
    Now by \eqref{eq:dD1} and \eqref{eq:dD2}, with probability $1-O(1/n)$ we have
    \begin{align*}
         \left\{\bigg|D^{(K)}_l-n\sum_{m=1}^M d_m(p_m^l)\bigg|< (\log n)^{1+\epsilon}n^{1/2}\right\}\cap \left\{\bigg|D^{(K)}_u-n\sum_{m=1}^M d_m(p_m^u)\bigg|< (\log n)^{1+\epsilon}n^{1/2}\right\}
    \end{align*}
    Since the order of the numerator and denominator of $\theta_0$ in \eqref{eq:theta0} is $n^{3/4}$,
    replacing $n\sum_{m=1}^Md_m(p_m^u)$ and $n\sum_{m=1}^Md_m(p_m^l)$ by $D^{(K)}_u$ and $D^{(K)}_l$, which are within an error bound of $(\log n)^{1+\epsilon}n^{1/2}$, does not significantly change the value of $\theta_0$.
    Thus, for a sufficiently large $n$ (we require $ n^{1/4}> (T-t_K)(\log n)^{1+\epsilon}$), we have
    \begin{align*}
        \theta = \frac{c - S(t_K)/n-(T-t_K)D^{(K)}_u}{(T-t_K)(D^{(K)}_l-D^{(k)}_u)}\in (0,1).
    \end{align*}
    with probability $\PR(E_K)=1-O(n^{-1/2})$.

    To prove~\eqref{eq:inventory-deviation}, note that conditional on $\mathcal F_{t_K}$, by Step~\ref{step:p_ml}, Step~\ref{step:p_mu}, Step~\ref{step:p_ml_theta} and Step~\ref{step:p_mu_theta} in the algorithm,  $S(T)-S({t_K})$ can be decomposed as
    \begin{align*}
        S(T)-S(t_K) = (\log n)^{-\epsilon} (D^{(K)}_l+D^{(K)}_u)+ N^{(K)}_l+ N^{(K)}_u,
    \end{align*}
    where $N^{(K)}_l$ and $N^{(K)}_u$ are Poisson random variables given $\theta$ with means
    \begin{equation*}
        n(\theta(T-t_K)-(\log n)^{-\epsilon})\sum_{m=1}^M d_m(p_m^l)\text{ and }n((1-\theta)(T-t_K)-(\log n)^{-\epsilon})\sum_{m=1}^M d_m(p_m^u)
    \end{equation*}
    respectively.\footnote{The variable $\theta$ is determined by $D^{(K)}_l$ and $D^{(K)}_u$, and thus measurable w.r.t. $\mathcal F_{t_K+2(\log n)^{-\epsilon}}$.}
    On the other hand, on the event $\theta\in (0,1)$, $nc-S(t_K)$ is equal to $(T-t_K)(\theta D^{(K)}_l+(1-\theta)D^{(K)}_u)$ by \eqref{eq:theta}.
    Therefore,
    \begin{align}
        \E[|S(T)-nc|] &\le \E\left[\left|(nc-S(t_K))- \left((\log n)^{-\epsilon} (D^{(K)}_l+D^{(K)}_u)+ N^{(K)}_l+ N^{(K)}_u\right)\right|\I{\theta\in(0,1)}\right]\notag\\
                      &\quad + \E[S(T)+nc]\PR(\theta\notin(0,1))\notag\\
                      &= \E\left[\left|(nc-S(t_K))- \left((\log n)^{-\epsilon} (D^{(K)}_l+D^{(K)}_u)+ N^{(K)}_l+ N^{(K)}_u\right)\right|\I{\theta\in(0,1)}\right]\notag\\
                      &\quad + O(\log(n) n\PR(\theta\notin(0,1)))\notag\\
                  & \le \E\left[\left| (\theta(T-t_K)-(\log n)^{-\epsilon})D^{(K)}_l- N^{(K)}_l\right|\right]\notag\\
                  &\quad +\E\left[\left|((1-\theta)(T-t_K)-(\log n)^{-\epsilon})D^{(K)}_u-N^{(K)}_u  \right|\right]+O(\log(n)n^{1/2})\label{eq:|ST|}
    \end{align}
    The equality follows from the argument below \eqref{eq:cond_exp_similar}: $S(T)$ can be expressed as the sum of $K+1$ Poisson random variables, conditional on $\mathcal F_{t_0}$, \dots, $\mathcal F_{t_K}$, $\mathcal F_{t_K+2(\log n)^{-\epsilon}}$ respectively.
    Therefore, applying the tower property we have $\E[S(T)]=O(\log(n)n)$ by the fact that $K+1=O(\log n)$.
    To bound the first term of \eqref{eq:|ST|}, we have
    \begin{align*}
        &\E\left[\left| (\theta(T-t_K)-(\log n)^{-\epsilon})D^{(K)}_l- N^{(K)}_l\right|\right]\\
        \le & \E\left[(\theta(T-t_K)-(\log n)^{-\epsilon})\left| D^{(K)}_l- n\sum_{m=1}^M d_m(p_m^l)\right|\right]\\
            &\quad+ \E\left[\left| N^{(K)}_l-n(\theta(T-t_K)-(\log n)^{-\epsilon})\sum_{m=1}^M d_m(p_m^l)\right| \right]\\
    \le & (T-t_K)\E\left[\E\left[\left| D^{(K)}_l- n\sum_{m=1}^M d_m(p_m^l)\right|\bigg| F_{t_K}\right]\right]\\
            &\quad+ \E\left[\E\left[\left| N^{(K)}_l-n(\theta(T-t_K)-(\log n)^{-\epsilon})\sum_{m=1}^M d_m(p_m^l)\right| \bigg| \theta\right]\right]\\
%            &\le (T-t_K) \E\left[\E\left[\left( D^{(K)}_l- n\sum_{m=1}^M d_m(p_m^l)\right)^2\bigg| F_{t_K}\right]^{1/2}\right]\\
%            &\quad+ \E\left[\E\left[\left( N^{(K)}_l-n(\theta(T-t_K)-(\log n)^{-\epsilon})\sum_{m=1}^M d_m(p_m^l)\right)^2 \bigg| \theta\right]^{1/2}\right]\\
    \le& (T-t_K) (\log n)^{\epsilon/2}\E\left[\left(  n\sum_{m=1}^M d_m(p_m^l)\right)^{1/2}\right]+\E\left[\left( n(\theta(T-t_K)-(\log n)^{-\epsilon})\sum_{m=1}^M d_m(p_m^l)\right)^{1/2}\right]\\
            = &O((\log n)^{\epsilon}n^{1/2}),
    \end{align*}
    In the last inequality, we have applied Lemma~\ref{lem:poisson-expect} to the centralized version of the Poisson random variables of $(\log n)^{-\epsilon} D_l^{(K)}$ and $N_l^{(K)}$.
    Similarly, we can bound the second term of \eqref{eq:|ST|}.
    This completes the proof of \eqref{eq:inventory-deviation}.

    To bound \eqref{eq:rate-deviation}, note that
    \begin{align}\label{eq:rate-two-terms}
        \E\left[\left|\int_{0}^T\sum_{m=1}^Md_m(P_m(t))dt-c\right|\right]&\le \E\left[\left|S(t_K)/n-\int_{0}^{t_K}\sum_{m=1}^Md_m(P_m(t))dt\right|\right]\notag\\
                                                                         &\quad + \E\left[\left|c-S(t_K)/n-\int_{t_K}^{T}\sum_{m=1}^Md_m(P_m(t))dt\right|\right]
    \end{align}
    The first term satisfies
    \begin{align*}
        &\E\left[\left|S(t_K)/n-\int_{0}^{t_K}\sum_{m=1}^Md_m(P_m(t))dt\right|\right] \\
        \le & \sum_{k=1}^{K-1} n^{-1}\E\left[\E\left[\left|(S(t_{k+1})-S(t_k))-n\int_{t_k}^{t_{k+1}}\sum_{m=1}^Md_m(P_m(t))dt\right|\bigg| F_{t_k}\right]\right]\\
        \le &\sum_{k=1}^{K-1}n^{-1/2}\E\left[ \int_{t_k}^{t_{k+1}}\sum_{m=1}^Md_m(P_m(t))dt\right]\\
        \le &\sum_{k=1}^{K-1} n^{-1/2} \sqrt{M_1\tau^{(k)}}.
    \end{align*}
    The second inequality applies Lemma~\ref{lem:poisson-expect} to the centered Poisson random variable $S(t_{k+1})-S(t_k)-n\int_{t_k}^{t_{k+1}}\sum_{m=1}^Md_m(P_m(t))dt$.
    The last inequality is due to the fact that $d_m(P_m(t))\le M_1$.
    Now, note that by the Cauchy-Schwartz inequality and Lemma~\ref{lem:K}
    \begin{align*}
        \sum_{k=1}^{K-1}\sqrt{\tau^{(k)}}\le (K\sum_{k=1}^{K-1}\tau^{(k)})^{1/2}=O((\log n)^{1/2}).
    \end{align*}
    Therefore, the first term of \eqref{eq:rate-two-terms} is $O((\log n)^{1/2}n^{-1/2})$.
    Next we bound the second term of \eqref{eq:rate-two-terms},
%    \begin{equation}\label{eq:prob-bound1}
%        \PR\left( \left|\int_{t_{K}}^T\sum_{m=1}^Md_m(P_m(t))dt-(c-S(t_K)/n)\right|\ge (\log n)^{2+\epsilon}n^{-1/2}\right)
%    \end{equation}
    Note that by the design of the algorithm, only $p_m^l$ and $p_m^u$ are used in phase $K$ for type-$m$ consumers and moreover
    \begin{align*}
        \int_{t_{K}}^T\sum_{m=1}^Md_m(P_m(t))dt =(T-t_K)\sum_{m=1}^M(\theta d_m(p_m^l) +(1-\theta)d_m(p^{u}_m)).
    \end{align*}
    In the proof of \eqref{eq:inventory-deviation}, we have shown that on the event $\theta\in(0,1)$, $nc-S(t_K)=(T-t_K)(\theta D^{(K)}_l+(1-\theta)D^{(K)}_u)$.
    Therefore, plugging in those quantities into the second term of \eqref{eq:rate-two-terms}, we have
    \begin{align*}
        &\E\left[\left|c-S(t_K)/n-\int_{t_K}^{T}\sum_{m=1}^Md_m(P_m(t))dt\right|\right]\\
        \le & \E\left[\left|c-S(t_K)/n-\int_{t_K}^{T}\sum_{m=1}^Md_m(P_m(t))dt\right|\I{\theta\in (0,1)}\right]+ (\E[S(t_K)/n]+O(1))\PR(\theta\notin(0,1))\\
        \le & \frac{1}{n} \E\left[\theta(T-t_K)\left|D_l^{(K)}-n\sum_{m=1}^Md_m(p_m^l)\right|\right]\\
            &\quad +\frac{1}{n} \E\left[(1-\theta)(T-t_K)\left|D_u^{(K)}-n\sum_{m=1}^Md_m(p_m^u)\right|\right]+O( n^{-1/2})\\
        \le & \frac{T}{n} \E\left[\left|D_l^{(K)}-n\sum_{m=1}^Md_m(p_m^l)\right|\right] +\frac{T}{n} \E\left[\left|D_u^{(K)}-n\sum_{m=1}^Md_m(p_m^u)\right|\right]+O( n^{-1/2})\\
        \le & O((\log n)^{\epsilon} n^{-1/2}).
    \end{align*}
    The last inequality is obtained by applying Lemma~\ref{lem:poisson-expect} to the centered version of the Poisson random variables $(\log n)^{-\epsilon}D_l^{(K)}$ and $(\log n)^{-\epsilon}D_u^{(K)}$.
%    \begin{align*}
%        \eqref{eq:prob-bound1}&\le \PR\left(\theta(T-t_K)\left|D_l^{(k)}-n\sum_{m=1}^Md_m(p_m^l)\right|> \frac{(\log n)^{\epsilon}n^{1/2}}{2}\right)\\
%                              &\quad +\PR\left((1-\theta)(T-t_K)\left|D_u^{(k)}-n\sum_{m=1}^Md_m(p_m^u)\right|> \frac{(\log n)^{2+\epsilon}n^{1/2}}{2}\right)
%        +\PR(\theta\notin (0,1))\\
%        &\le \PR\left((T-t_K)\left|D_l^{(k)}-n\sum_{m=1}^Md_m(p_m^l)\right|> \frac{(\log n)^{2+\epsilon}n^{1/2}}{2}\right)\\
%        &\quad +\PR\left((T-t_K)\left|D_u^{(k)}-n\sum_{m=1}^Md_m(p_m^u)\right|> \frac{(\log n)^{\epsilon}n^{1/2}}{2}\right)+\PR(\theta\notin (0,1))\\
%        &=O(n^{-1/2}).
%    \end{align*}
%    The last equality is because of \eqref{eq:dD1} and \eqref{eq:dD2}, as well as $\PR(\theta\notin (0,1))=O(n^{-1/2})$ proved before.
    This completes the proof.
\end{proof}

\begin{proof}[Proof of Proposition~\ref{prop:regret-z>0}:]
    Similar to \eqref{eq:z<0_all_regret} in the proof of Proposition~\ref{prop:regret-z<=0},
    the total revenue of $\tilde{J}_n^{\pi}$ can be expressed as
    \begin{align*}
        \E\left[\tilde{J}_{n}^{\pi}\right]= \sum_{k=1}^K\E[\tilde{J}_{n,k}^{\pi}] - \overline p\E\left[ \left(S(T)-nc\right)^+\right].
    \end{align*}
    By Lemma~\ref{lem:inventory-XT}, $\overline p\E[(S(T)-nc)^+]\le \overline p\E[|S(T)-nc|]=O((\log n)^{\epsilon}n^{1/2})$.
    The first term can be expressed as
    \begin{align*}
        \sum_{k=1}^K\E[\tilde{J}_{n,k}^{\pi}] &=\E\left[n\int_{0}^{T}\sum_{m=1}^MP_m(t)d_m(P_m(t))dt\right]\\
                                              &\ge \E\left[n\int_{0}^{T}\sum_{m=1}^MP_m(t)d_m(P_m(t))dt\I{\cap_{k=1}^KA_k\cap B_K}\right]\\
         &=  nT\sum_{m=1}^Mp_m^{\ast}d_m(p_m^{\ast})\PR(\cap_{k=1}^KA_k\cap B_K)\\
          &\quad -\E\left[n\int_{0}^{T}\sum_{m=1}^M(p_m^{\ast}d_m(p_m^{\ast})-P_m(t)d_m(P_m(t)))dt\I{\cap_{k=1}^KA_k\cap B_K}\right]\\
          & \ge   J_n^D(T,c) - O(n^{1/2})\\
          &\quad -\E\left[n\int_{0}^{T}\sum_{m=1}^M(p_m^{\ast}d_m(p_m^{\ast})-P_m(t)d_m(P_m(t)))dt\I{\cap_{k=1}^KA_k\cap B_K}\right]
    \end{align*}
    because $\PR(\cap_{k=1}^KA_k\cap B_K)=O(n^{-1/2})$ by Lemma~\ref{lem:bi-ci}.
    Define a stochastic process $\xi_m(t)\triangleq d_m(p_m^{\ast})-d_m(P_m(t))$, and define $r^{\ast\prime}_m$ and $r^{\ast\prime\prime}_m$ to be the first- and second-order derivative of $r_m(\lambda)=\lambda d^{-1}_m(\lambda)$ at $\lambda=d_m(p_m^{\ast})$.
    Note that in the case of $z^{\ast}>0$, $d_m(p_m^{\ast})$ is not the unconstrained maximizer of $r_m(\cdot)$.
    By Taylor's expansion, we have
    \begin{align}\label{eq:taylor}
        &\E\left[n\int_{0}^{T}\sum_{m=1}^M(p_m^{\ast}d_m(p_m^{\ast})-P_m(t)d_m(P_m(t)))dt\I{\cap_{k=1}^KA_k\cap B_K}\right]  \notag \\
        \le & \E\left[n\int_{0}^{T}\sum_{m=1}^M(r^{\ast\prime}_m\xi_m(t)+|r^{\ast\prime\prime}_m|\xi_m(t)^2)dt\I{\cap_{k=1}^KA_k\cap B_K}\right]\notag\\
        \le & \E\left[n\int_{0}^{T}\sum_{m=1}^Mr^{\ast\prime}_m\xi_m(t)dt\I{\cap_{k=1}^KA_k\cap B_K}\right]+M_4 \E\left[n\int_{0}^{T}\sum_{m=1}^M(\xi_m(t))^2dt\I{\cap_{k=1}^KA_k\cap B_K}\right].
    \end{align}
    Next we show that $r^{\ast\prime}_1=r^{\ast\prime}_2=\ldots=r^{\ast\prime}_M=z^{\ast}$. Indeed, $p^{\ast}_m$ maximizes $\mathcal R_m(z^{\ast})$ by Proposition~\ref{prop:dual}. Equivalently, $\lambda = d_m(p_m^{\ast})$ maximizes $r_m(\lambda)-\lambda z^{\ast}$. The first-order condition implies
    $r^{\ast\prime}_m = z^{\ast}$.
    Therefore, the first term in \eqref{eq:taylor} is bounded above by
    \begin{align*}
        &\E\left[n\int_{0}^{T}\sum_{m=1}^Mr^{\ast\prime}_m\xi_m(t)dt\I{\cap_{k=1}^KA_k\cap B_K}\right]\\
        \le & nz^{\ast} \E\left[ \int_{0}^{T}\sum_{m=1}^M \left(d_m(p_m^\ast)- d_m(P_m(t))\right)dt\I{\cap_{k=1}^KA_k\cap B_K}\right]\\
        \le & nz^{\ast} \E\left[\left|c- \int_{0}^{T}\sum_{m=1}^M d_m(P_m(t))dt\right|\right].
    \end{align*}
    By Lemma~\ref{lem:inventory-XT}, the above term is $O((\log n)^{\epsilon}n^{1/2})$.

    For the second term in \eqref{eq:taylor}, note that
    on the event $\cap_{k=1}^KA_K\cap B_K$, we have that for all $m$
    \begin{align*}
        &\left|P_m(t)-p_m^{\ast}\right|\le \bar\Delta^{(k)}\quad \forall t\in(t_k,t_{k+1}],\; k=1,\ldots,K-1\\
        \implies &|\xi_m(t)|\le M_2\bar\Delta^{(k)}\quad \forall t\in(t_k,t_{k+1}],\; k=1,\ldots,K-1
    \end{align*}
    and
    \begin{align*}
        &\left|P_m(t)-p_m^{\ast}\right|\le \bar\Delta^{(K)}+(\log n)^{1+9\epsilon} n^{-1/4}\le 2(\log n)^{1+9\epsilon} n^{-1/4}\quad \forall t\in(t_K,T]\\
        \implies &|\xi_m(t)|\le 2M_2(\log n)^{1+9\epsilon} n^{-1/4}\quad \forall t\in(t_K,T]
    \end{align*}
    Therefore, by the choice of $\tau^{(k)}$ and $\bar\Delta^{(k)}$,
    \begin{align*}
        &\E\left[n\int_{0}^{T}\sum_{m=1}^M(\xi_m(t))^2dt\I{\cap_{k=1}^KA_k\cap B_K}\right]\le n\sum_{k=1}^{K-1}\tau^{(k)}(\bar\Delta^{(k)})^2+ 4M_2^2\tau^{(K)}(\log n)^{2+18\epsilon}n^{1/2}\\
        \le &(K-1)(\log n)^{1+15\epsilon}n^{1/2}+4M_2^2\tau^{(K)}(\log n)^{2+18\epsilon}n^{1/2}\\
        =& O((\log n)^{2+18\epsilon}n^{1/2}).
    \end{align*}
    This completes the proof.
\end{proof}
\end{appendices}

\end{document}